\documentclass{article}

    \PassOptionsToPackage{numbers, compress}{natbib}
 \usepackage[preprint]{neurips_2026}

\usepackage[utf8]{inputenc} 
\usepackage[T1]{fontenc}    
\usepackage{hyperref}       
\usepackage{url}            
\usepackage{booktabs}       
\usepackage{amsfonts}       
\usepackage{nicefrac}       
\usepackage{microtype}      
\usepackage{xcolor}         

\usepackage{amsmath}
\usepackage{amssymb}
\usepackage{amsthm} 
\usepackage{abbreviations}
\usepackage{algorithm}    
\usepackage{algpseudocode}
\usepackage{graphicx}
\usepackage{subcaption}

\usepackage{thmtools}

\usepackage{hyperref} 

\addtocontents{toc}{\protect\setcounter{tocdepth}{0}}

\title{Block Optimism for Nonstationary Bandits with Latent Linear Dynamics}

\author{
Taehyun Hwang \qquad Hyunjun Choi \qquad  Heesang Ann \qquad  Min-hwan Oh
\\
\\
Seoul National University
\\
Soeul, South Korea
\\
\{\texttt{th.hwang, nschj1, sang3798, minoh}\}\texttt{@snu.ac.kr}
}

\begin{document}

\maketitle

\begin{abstract}
We study an endogenous nonstationary stochastic bandit problem with latent linear dynamics, where actions affect both immediate rewards and the future evolution of an unobserved latent state.
Rewards are bilinear in the current action and latent state, inducing history-dependent rewards and a nontrivial long-horizon planning problem.
The existing explore-then-commit approach achieves $\tilde{\mathcal{O}}(T^{2/3})$ regret by uniformly exploring to estimate the latent dynamics and then committing to an optimized open-loop action sequence.
We show that this rate can be improved via adaptive block-level optimism.
Our key step is a cyclic approximation: under stable dynamics, the infinite-memory reward process can be truncated, and the open-loop benchmark can be approximated by optimizing a finite-memory block-level proxy.
Building on this reduction, we propose a UCB-based block algorithm that maintains confidence sets for the truncated dynamics parameters and selects blocks optimistically.
We prove a regret bound of order $\tilde{\mathcal{O}}(\sqrt T)$, significantly improving over the previous $\tilde{\mathcal{O}}(T^{2/3})$ guarantee for the same model.
To the best of our knowledge, this is the first $\tilde{\mathcal{O}}(\sqrt T)$ regret guarantee for latent linear-dynamics bandits with bilinear reward observations and an open-loop action-sequence benchmark.
\end{abstract}

\section{Introduction}
Many sequential decision-making problems exhibit \emph{endogenous nonstationarity}: 
the learner's actions affect not only immediate rewards but also the future state of the environment.
This phenomenon naturally arises in recommender systems, online advertising, and personalized interventions, where repeated recommendations or interventions can change a user's latent preferences and thereby alter future rewards~\citep{kapoor2015just,kunaver2017diversity,kovacs2018rotating,schedl2018current}.
While classical stochastic bandits assume fixed reward distributions, much of the nonstationary bandit literature models nonstationarity as exogenous, for example through variation budgets, change points, or drifting parameters~\citep{besbes2014stochastic,cheung2019learning,russac2019weighted,kim2020randomized}.
Other action-induced models, such as rotting, recharging, blocking, rebounding, and delay-dependent bandits~\citep{levine2017rotting,seznec2019rotting,kleinberg2018recharging,basu2019blocking,leqi2021rebounding,cella2020stochastic}, capture temporal effects at the arm level, but are less suited for structured action spaces where one action can affect the future rewards of many related actions.

A recent work introduces a latent linear-dynamics bandit model in which actions drive an unobserved latent state with linear dynamics and rewards are bilinear in the current action and latent state~\citep{choi2026explorethencommit}.
This model captures cross-action endogenous effects through a latent state representation, but creates a coupled learning-and-planning problem: the learner must infer the latent input-output dynamics from scalar, temporally correlated rewards while choosing actions whose effects persist over the horizon.
The prior explore-then-commit (ETC) algorithm separates these tasks by first estimating the latent dynamics from random exploration and then committing to an optimized open-loop action sequence, yielding a \(\widetilde{\Ocal}(T^{2/3})\) regret guarantee~\citep{choi2026explorethencommit}.
This phase separation leads to a suboptimal exploration--exploitation trade-off, since data collected during exploration are not used for reward maximization, while observations obtained after commitment are not used to revise the plan.
This raises a natural question:
\begin{center}
\textit{Can we design an adaptive exploration strategy for latent linear-dynamics bandits that interleaves learning and planning and achieves sharper regret?}
\end{center}
  
We answer this question affirmatively by exploiting the stability of the latent dynamics.
In latent linear-dynamics bandits, past actions can continue to influence future rewards through the evolving latent state, so the induced reward process generally depends on the full action history.
Under stable dynamics, however, the influence of far-past actions decays geometrically, allowing this infinite-memory reward process to be approximated by a finite-memory model.
This observation enables a cyclic approximation that reduces the original long-horizon planning problem to a stationary block-level learning problem.
In this block-level problem, the proxy reward is approximately linear in a truncated representation of the latent dynamics.
Building on this reduction, we develop an adaptive optimistic learning strategy that updates its estimate at the end of each block and selects subsequent blocks optimistically.
This interleaves learning and planning throughout the horizon and yields a regret bound of order \(\widetilde{\Ocal}(\sqrt T)\), significantly improving over the \(\widetilde{\Ocal}(T^{2/3})\) ETC guarantee for the same model.    
Our main contributions are summarized as follows:
\begin{itemize}
    \item
    We first show that, under stable latent dynamics, the influence of far-past actions decays geometrically, allowing the infinite-memory reward process to be approximated by truncating its dependence to a finite window of recent actions. 
    Building on this approximation, we extend the cyclic approximation principle for finite-memory bandits~\citep{clerici2024linear, ann2026generalized} to our latent linear-dynamics setting, showing that the optimal open-loop benchmark can still be approximated by a cyclic block sequence despite the persistent influence of past actions through the latent state.

    \item 
    Building on this block-level reduction, we propose a block-level UCB algorithm that adaptively explores the latent dynamics while exploiting the current estimate.
    Unlike the previous ETC method~\citep{choi2026explorethencommit}, our algorithm does not require a long initial phase of uniform exploration.
    Instead, it maintains confidence sets for the truncated dynamics parameters, selects each block optimistically, and updates its estimate after every block, so that every block contributes simultaneously to exploration and exploitation.

    \item
    We prove that the proposed algorithm achieves a
    \(\widetilde{\Ocal}(\sqrt T)\)-order regret bound.
    We decompose each block-level reward observation into a linear signal and residual terms.
    The main challenge is that one residual component contains pre-block process noises and can therefore be correlated with the adaptively selected block, so it cannot be treated as standard martingale noise.
    We address this dependence by establishing a uniform bound on the corresponding vector-valued tail, which yields valid confidence sets under adaptive block selection and enables the optimism-based regret analysis.
    This significantly improves over the previous
    \(\widetilde{\Ocal}(T^{2/3})\) ETC guarantee for the same latent linear-dynamics bandit model~\citep{choi2026explorethencommit}.
    To the best of our knowledge, this is the first
    \(\widetilde{\Ocal}(\sqrt T)\)-order regret guarantee for latent linear-dynamics bandits with bilinear reward observations.

    \item
    We provide numerical experiments demonstrating that the proposed block-level optimistic algorithm performs favorably against baselines.
    The results illustrate the empirical benefit of adaptive block-level exploration.
\end{itemize}

\subsection{Related Work}
\paragraph{Nonstationary bandits.}
One common approach models nonstationarity as \textit{exogenous}, for instance through variation budgets~\citep{besbes2014stochastic}, switching or drifting reward parameters~\citep{cheung2019learning,russac2019weighted}, or nonstationary stochastic linear bandit models~\citep{kim2020randomized}.
These models capture environments whose reward functions change over time, but the changes are typically not caused by the learner's own actions.

Several lines of work study action-induced endogenous nonstationarity.
Rotting bandits~\citep{levine2017rotting,seznec2019rotting}, recharging bandits~\citep{kleinberg2018recharging}, blocking bandits~\citep{basu2019blocking}, rebounding bandits~\citep{leqi2021rebounding}, and bandits with delay-dependent payoffs~\citep{cella2020stochastic} model temporal effects such as satiation, recovery, blocking, or delay-dependent rewards.
These models typically describe how the reward of an individual arm changes as a function of its own past usage or delay.
Another line of work studies bandit models with finite memory, where the reward at each round depends on only a bounded window of past actions~\citep{clerici2024linear,ann2026generalized}.
In these models, the memory length is specified as part of the problem, so the influence of sufficiently distant actions vanishes by construction.

A closer structural analogue is dynamical linear bandits~\citep{mussi2023dynamical}, where actions drive an unobserved linear dynamical system.
Their reward is a linear combination of the latent state and the current action, and under stability the relevant benchmark reduces to the best constant action.
By contrast, the bilinear reward structure in our setting induces interactions between current and past actions, leading to a finite-horizon open-loop action-sequence benchmark.
The closest setting to ours is the latent linear-dynamics bandit model~\citep{choi2026explorethencommit}, which adopts this bilinear reward structure.
Unlike bandit models with finite memory, the resulting reward process has infinite memory because past actions can continue to influence future rewards through the latent state.
Compared with the previous ETC approach~\citep{choi2026explorethencommit}, we show that a UCB-based adaptive algorithm achieves improved statistical efficiency, yielding a sharper regret guarantee.

\paragraph{Online control and nonstochastic control.}
Our problem is also related to online control of linear dynamical systems.
A large literature studies regret minimization in \textit{Linear Quadratic Regulator} (LQR), \textit{Linear Quadratic Gaussian} (LQG), and related control problems, where the learner controls a linear system and incurs state-action costs~\citep{abbasi2011regret,dean2018regret,mania2019certainty,cohen2019learning,lale2021adaptive}.
More recent work on nonstochastic control allows adversarial perturbations and adversarially chosen convex losses~\citep{agarwal2019online,agarwal2019logarithmic,hazan2020nonstochastic,simchowitz2020improper}.
These works share with ours the use of stability and input-output representations of linear dynamical systems.
However, their objective is control with exogenous loss functions, typically convex or quadratic costs in the state, observation, and control, and the benchmark is a feedback controller class.
By contrast, our model is a bandit reward-maximization problem in which the only observation is a scalar bilinear reward, and the benchmark is an optimal finite-horizon open-loop action sequence.
This bilinear reward structure creates pairwise interactions between current and past actions after unrolling the dynamics, which do not arise in standard LQR/LQG or nonstochastic control formulations.

\section{Preliminaries}
\paragraph{Notations.}
For a vector \(\xb \in \RR^d\), we denote its $\ell_2$-norm by \(\|\xb\|_2\).
For a positive semidefinite matrix \(\Ab\), we define the \(\Ab\)-weighted norm by
\(\|\xb\|_{\Ab} = \sqrt{\xb^\top \Ab \xb}\).
For a matrix \(\Mb\), we use \(\|\Mb\|_2\) to denote its spectral norm, \(\rho(\Mb)\) its spectral radius, and \(\lambda_{\max}(\Mb)\) its largest eigenvalue when \(\Mb\) is symmetric.
We denote the Frobenius norm of a matrix \(\Mb\) by \(\|\Mb\|_F\).
We write \(\vect(\Mb)\) for the vectorization of \(\Mb\), and \(\otimes\) for the Kronecker product.
The vector \(\zero_d\) denotes the zero vector in \(\mathbb R^d\).
We use \(\Ocal(\cdot)\) to hide universal constant factors, and \(\widetilde{\Ocal}(\cdot)\) to additionally hide logarithmic factors.

\paragraph{Problem setup.}
We study an endogenous nonstationary stochastic bandit problem in which rewards are generated through controlled latent dynamics.
At each round \(t=1,\ldots,T\), the learner selects an action \(\ub_t \in \Ucal \subset \RR^p\) and observes a scalar reward \(r_t\), while the latent state \(\xb_t \in \RR^n\) remains unobserved.
The reward depends bilinearly on the current action and the latent state, and the action also affects the future evolution of the latent state.
Specifically, the system evolves as
\begin{equation}
    \xb_{t+1} = \Ab \xb_t + \Bb \ub_t + \wb_t,
    \qquad
    r_t = \ub_t^\top \Cb \xb_t + z_t,
    \label{eq:latent-dynamics}
\end{equation}
where \(\xb_1=\zero_n\), and \(\Ab\in\RR^{n\times n}\), \(\Bb\in\RR^{n\times p}\), and \(\Cb\in\RR^{p\times n}\) are \textit{unknown} matrices.
The process noise \(\{\wb_t\}_{t\ge1}\) and reward noise \(\{z_t\}_{t\ge1}\) are unobserved zero-mean noise processes.

For a deterministic open-loop action sequence \(\ub_{1:T}\in\Ucal^T\),
unrolling Eq.~\eqref{eq:latent-dynamics} gives
$
    r_t
    =
    \sum_{s=1}^{t-1}
    \ub_t^\top \Cb \Ab^{s-1}\Bb \ub_{t-s}
    +
    \sum_{s=1}^{t-1}
    \ub_t^\top \Cb \Ab^{s-1}\wb_{t-s}
    +
    z_t .
    \label{eq:unrolled-reward}
$
Therefore, for deterministic open-loop sequences, the expected reward at round
\(t\) is
\begin{equation}
    \mu_t(\ub_{1:t})
    =
    \sum_{s=1}^{\infty}
    \ub_t^\top \Gb_{s-1}\ub_{t-s}.
    \qquad
    \Gb_h := \Cb \Ab^h \Bb \, ,
    \label{eq:mean-reward}
\end{equation}
where we set \(\ub_s=\zero_p\) for all \(s\le0\).
For adaptive policies, we use \(\mu_t(\ub_{1:t})\) as the action-driven
mean-reward functional obtained by removing the process-noise and reward-noise
terms from the unrolled reward.
The open-loop benchmark is
$
    \opt_T
    :=
    \sup_{\ub_{1:T}\in\Ucal^T}
    \sum_{t=1}^T \mu_t(\ub_{1:t}).
$
For an adaptive policy \(\pi\), the realized action sequence
\(\ub_{1:T}^\pi\) is random. We define its \textit{pseudo-regret} with respect to the
open-loop benchmark as
\begin{equation*}
    \Rcal^{\mathrm{pseudo}}_T(\pi)
    :=
    \opt_T
    -
    \sum_{t=1}^T \mu_t(\ub_{1:t}^\pi).
    \label{eq:pseudo-regret}    
\end{equation*}
Thus, \(\Rcal^{\mathrm{pseudo}}_T(\pi)\) is a random variable through the action sequence
generated by the policy.
Now, we define the actual expected regret in terms of the realized rewards
collected by policy \(\pi\):
\begin{equation*}
    \regret_T(\pi)
    :=
    \opt_T
    -
    \EE_\pi\!\left[
        \sum_{t=1}^T r_t^\pi
    \right] .
    \label{eq:actual-regret}
\end{equation*}
Here the expectation is over the policy randomization and the process and reward
noises.

\begin{remark}
    Recent work on latent linear dynamics bandits~\citep{choi2026explorethencommit}
    analyzes expected regret defined in terms of realized rewards.
    In this paper, we first establish a high-probability bound on the pseudo-regret
    \(\Rcal^{\mathrm{pseudo}}_T(\pi)\), which is defined with respect to the action-driven
    mean-reward functional. For adaptive policies, this pseudo-regret need not
    coincide with actual expected regret, because actions may be correlated with
    past process noise through previously observed rewards. We then translate the
    pseudo-regret guarantee into an actual expected regret guarantee for our
    block-committed algorithm by bounding the resulting adaptation gap (Lemma~\ref{lem:blockwise-adaptation-gap}).
\end{remark}

\paragraph{Assumptions.}
We next state some standard modeling assumptions, commonly used in stochastic bandits and linear dynamical systems, under which we design and analyze our algorithm.

\begin{assumption}[Action set] \label{assm:bounded-action}
    The action set is compact, i.e., $\Ucal$ is closed and for all $\ub \in \Ucal$, $\| \ub \|_2 \le 1$.
\end{assumption}

\begin{assumption}[Stability] \label{assm:stability}
    For some $\rho\in(0,1)$ and $M_\rho\ge 1$, 
    $\|\Ab^h\|_2 \le M_\rho \rho^h$ for all $h\ge 0$.    
\end{assumption}

\begin{assumption}[Sub-Gaussian noise]
\label{assm:noise}
    Let \(\{\Fcal_t\}_{t\ge0}\) be a filtration containing the interaction history and the latent process. At each round \(t\), the action \(\ub_t\) is \(\Fcal_{t-1}\)-measurable, and the noises \((z_t,\wb_t)\) are realized after
    \(\ub_t\) is chosen.
    Conditional on \(\Fcal_{t-1}\), \(z_t\) and \(\wb_t\) are independent.
    The reward noise \(z_t\) is conditionally \(\sigma_z^2\)-sub-Gaussian, and the
    process noise \(\wb_t\) is conditionally sub-Gaussian with variance proxy
    \(\Sigma_w\succeq \zero_{n\times n}\): for every \(\vb\in\RR^n\) and every
    \(\alpha\in\RR\),
    \begin{equation*}
        \mathbb E\!\left[\exp(\alpha z_t)\mid \Fcal_{t-1}\right]
        \le
        \exp\!\left(\frac{\alpha^2\sigma_z^2}{2}\right) 
        \quad \text{and} \quad
        \mathbb E\!\left[\exp(\alpha \vb^\top \wb_t)\mid \Fcal_{t-1}\right]
        \le
        \exp\!\left(\frac{\alpha^2 \vb^\top \Sigmab_w \vb}{2}\right) \, .
    \end{equation*}
\end{assumption}

These assumptions are standard in stochastic bandit and linear dynamical system models. 
Assumption~\ref{assm:bounded-action} rules out unbounded rewards and ensures that the open-loop benchmark is well-defined over a compact action set. 
Assumption~\ref{assm:stability} imposes exponential stability of the latent dynamics, which is equivalent to Schur stability in finite dimensions. 
This type of stability assumption is common in learning and control of linear dynamical systems, since it keeps the state and rewards well behaved under persistent inputs and process noise, and ensures geometric decay of past actions and noises~\citep{simchowitz2020improper, oymak2021revisiting, lee2022improved, sun2023optimal, sattar2025learning, choi2026explorethencommit}. 
Assumption~\ref{assm:noise} is a standard light-tail condition used for high-probability least-squares and self-normalized martingale concentration arguments.

\section{Main Results}
\subsection{Cyclic Approximation}
A central difficulty in the latent-dynamics bandit model is that learning and planning are inherently coupled.
The learner must estimate the latent input-output dynamics from temporally correlated rewards, while also choosing actions whose effects propagate through the latent state over the horizon.
The ETC approach~\citep{choi2026explorethencommit} addresses this challenge by separating learning and planning: it first estimates the latent dynamics from exploratory data and then commits to an optimized open-loop sequence.
This phase separation yields a \(\widetilde{\Ocal}(T^{2/3})\) regret rate, reflecting a suboptimal exploration--exploitation trade-off.
Our goal is to avoid this separation by updating the dynamics estimate and revising the plan adaptively throughout the horizon.

To avoid this rigid phase separation, we exploit the stability of the latent
dynamics. Stability makes the effect of far-past actions decay geometrically,
allowing the infinite-memory mean-reward functional to be approximated by a
finite-memory proxy. We then use a block construction of length \(m+L\): the
first \(m\) actions serve as a warm-up segment, and the last \(L\) positions yield
informative rewards whose \(m\)-step histories remain inside the block. This
turns the long-horizon open-loop problem into a stationary block-level learning
problem, in which each block is a composite action and the proxy reward is linear
in the truncated Markov parameters.

\paragraph{Finite-memory approximation.}
We first formalize the finite-memory approximation of the mean reward defined in
Eq.~\eqref{eq:mean-reward}. For a memory cutoff \(m\ge 1\), define the
truncated mean reward as
\begin{equation*}
    \bar\mu_t(\ub_{1:t})
    :=
    \sum_{s=1}^{m}
    \ub_t^\top \Gb_{s-1}\ub_{t-s} \, ,
\end{equation*}
where we use the convention \(\ub_j=\zero_p\) for all \(j\le 0\).
Under the stability condition in Assumption~\ref{assm:stability}, the truncation
error decays geometrically. In particular, since \(\|\ub\|_2\le 1\) for every
\(\ub\in\Ucal\), we have
\begin{equation}
    \left|
        \mu_t(\ub_{1:t}) - \bar\mu_t(\ub_{1:t})
    \right|
    \le
    \Delta_m
    :=
    \sum_{h=m}^{\infty}\|\Gb_h\|_2
    \le
    \frac{
        \|\Cb\|_2\|\Bb\|_2 M_\rho
    }{
        1-\rho
    }
    \rho^m \, .
    \label{eq:truncation-error}
\end{equation}
We refer to \(\Delta_m\) as the finite-memory truncation error. The derivation
of Eq.~\eqref{eq:truncation-error} is provided in
Lemma~\ref{lem:finite-memory-truncation-error} in
Appendix~\ref{appx:cyclic-approximation}.    

\paragraph{Block-level proxy reward.}
The finite-memory representation allows us to localize the reward within short blocks.
Fix an informative suffix length \(L\ge 1\), and let the block length be \(m+L\).
For simplicity, we assume that \(T/(m+L)\) is an integer; a nonzero remainder contributes only an additive \(\Ocal(m+L)\) term.

For a block \(\Ub = (\vb_1,\ldots,\vb_{m+L}) \in \mathcal U^{m+L}\),
we define its proxy reward as the truncated reward accumulated over the last \(L\) positions of the block as
$
\bar r_m(\Ub)
    :=
    \sum_{j=m+1}^{m+L}
    \sum_{s=1}^{m}
    \vb_j^\top \Gb_{s-1}\vb_{j-s} \, .    
$
The first \(m\) positions serve as a warm-up segment so that every informative position \(j \in \{m+1,\ldots,m+L\}\) has its full \(m\)-step action history contained inside the same block.
This block construction reveals a \textit{linear reward structure}: \(\bar r_m(\Ub)\) can be written as a linear function of the truncated Markov parameters.
Define
\[
    \thetab_m^*
    :=
    \begin{bmatrix}
        \mathrm{vec}(\Gb_0)\\
        \mathrm{vec}(\Gb_1)\\
        \vdots\\
        \mathrm{vec}(\Gb_{m-1})
    \end{bmatrix}
    \in \mathbb R^{p^2m},
    \qquad
    \phib_j(\Ub)
    :=
    \begin{bmatrix}
        \vb_{j-1}\otimes \vb_j\\
        \vb_{j-2}\otimes \vb_j\\
        \vdots\\
        \vb_{j-m}\otimes \vb_j
    \end{bmatrix}
    \in \mathbb R^{p^2m},
\]
for each informative position \(j\in \{m+1,\ldots,m+L\}\).
Then we have
\begin{equation}
    \bar r_m(\Ub)
    \!
    =
    \!\!\!\!\sum_{j=m+1}^{m+L}
    \sum_{s=1}^{m}
    \vb_j^\top \Gb_{s-1}\vb_{j-s} 
    =
    \!\!\!\!\sum_{j=m+1}^{m+L}
    \sum_{s=1}^{m}
    \ips{\vb_{j-s} \otimes \vb_j}{\vect(\Gb_{s-1})}
    =
    \!\!\!\!\sum_{j=m+1}^{m+L}
    \left\langle
        \phib_j(\Ub),
        \thetab_m^*
    \right\rangle .
    \label{eq:block-linear-reward}
\end{equation}
Thus, under the finite-memory block construction, each block can be viewed as a composite action whose expected proxy reward is linear in the unknown parameter \(\thetab_m^*\).
This linear representation is the key reduction that enables us to apply optimism-based linear bandit methods at the block level.

\paragraph{Cyclic approximation.}
We next show that the original open-loop benchmark can be approximated by optimizing the block-level proxy reward.
Let \(\Ub^*\in\argmax_{\Ub\in\Ucal^{m+L}}\bar r_m(\Ub)\) be a \textit{proxy-optimal block}, and let \(\Ub^*_\mathrm{cyc}\in\Ucal^T\) denote its cyclic repetition over the horizon \(T\).
The following proposition shows that this cyclic sequence approximates the optimal benchmark up to a controlled additive error.

\begin{restatable}[Cyclic approximation]{proposition}{CyclicApproximation}
\label{prop:cyclic-approximation}
    Suppose that Assumptions~\ref{assm:bounded-action} and~\ref{assm:stability} hold. 
    Then, we have
    \begin{equation*}
        \sup_{\ub_{1:T}\in\mathcal U^T}
        \sum_{t=1}^T
        \bar\mu_t(\ub_{1:t})
        -
        \sum_{t=1}^T
        \bar\mu_t\!\left(
            (\Ub^*_\mathrm{cyc})_{1:t}
        \right)
        \le
        \frac{2Tm}{m+L}R_m \, ,
        \label{eq:truncated-cyclic-approx}
    \end{equation*}
    where $R_m := \frac{\|\Cb\|_2 \|\Bb \|_2 M_\rho}{1-\rho}$.
    Consequently,
    \begin{equation*}
        \mathrm{OPT}_T
        -
        \sum_{t=1}^T
        \mu_t\!\left(
            (\Ub^*_\mathrm{cyc})_{1:t}
        \right)
        \le
        2T\Delta_m
        +
        \frac{2Tm}{m+L}R_m \, ,
        \label{eq:true-cyclic-approx}
    \end{equation*}
    where \(\Delta_m\) is the finite-memory truncation error defined in Eq.~\eqref{eq:truncation-error}.
\end{restatable}

\paragraph{Discussion of Proposition~\ref{prop:cyclic-approximation}.}
Proposition~\ref{prop:cyclic-approximation} reduces the original nonstationary planning problem to a stationary block-level optimization problem: it suffices to learn and optimize the proxy reward \(\bar r_m(\Ub)\) over blocks, while paying the approximation error \(2T\Delta_m+2TmR_m/(m+L)\).

This reduction is related to the block-cyclic construction for linear bandits with finite memory~\citep{clerici2024linear, ann2026generalized}, but our setting differs in two important ways.
First, in finite-memory bandits, the memory length is part of the model, so after ignoring the first \(m\) positions of each block, the remaining rewards are exactly localized within the block.
In our latent-dynamics model, the reward process has \textit{infinite memory}: every past action can affect the current reward through the latent state, and the cutoff \(m\) is only an approximation parameter.
Thus, our block reduction must control both the truncation error \(\Delta_m\) and the warm-up loss induced at block boundaries, of order \(mR_m/(m+L)\).
Second, the reward in finite-memory linear bandits~\citep{clerici2024linear, ann2026generalized} is linear in observed action features, whereas our reward depends bilinearly on the current action and an unobserved latent state.
After unrolling the dynamics, this induces a quadratic dependence on pairs of past and current actions through unknown Markov parameters.
Therefore, Proposition~\ref{prop:cyclic-approximation} is not merely a reuse of a finite-memory blocking argument; it shows that an endogenous latent-dynamics bandit with infinite memory can be approximated by a stationary linear bandit problem over blocks.

\subsection{Algorithm: $\algname$}
\begin{algorithm}[t!]
\caption{\(\algname\)}
\label{alg:main-algorithm}
\begin{algorithmic}[1]
    \State \textbf{Inputs:} horizon \(T \in \NN\), memory cutoff \(m \in \NN\), informative suffix length \(L \in \NN\), regularization parameter \(\lambda > 0\), confidence radii \(\{\gamma_{\tau}\}_{\tau\ge0}\)
    \State Set \(d = p^2m\), \(N=\lceil T/(m+L)\rceil\), and initialize \(\Vb_0 = \lambda \Ib_d\), \(\hat{\thetab}_0 = \zero_d\).
    \For{\(\tau=1,2,\dots,N\)}
        \State 
        Choose 
        $   \displaystyle
            \Ub_\tau\in\argmax_{\Ub\in\Ucal^{m+L}}
            \sum_{j=m+1}^{m+L}
            \Big(
                \ips{\phib_j(\Ub)}{\hat{\thetab}_{\tau-1}}
                +
                \gamma_{\tau-1}\,\norm{\phib_j(\Ub)}_{\Vb_{\tau-1}^{-1}}
            \Big)
        $
        \State
        Play \(\Ub_\tau=(\ub_{\tau,1},\ldots,\ub_{\tau,m+L})\) sequentially and observe rewards \(r_{\tau,1},\ldots,r_{\tau,m+L}\)
        \State Update
        $\displaystyle
            \Vb_\tau
            =
            \Vb_{\tau-1}
            +
            \sum_{j=m+1}^{m+L}
            \phib_j(\Ub_\tau)\phib_j(\Ub_\tau)^\top \, ,
        $
        $   \displaystyle
            \hat{\thetab}_\tau 
            = \Vb_\tau^{-1}
            \bigg(
                \sum_{\tau'=1}^{\tau}
                \sum_{j=m+1}^{m+L}
                \phib_j(\Ub_{\tau'})
                r_{\tau',j}
            \bigg)
        $
    \EndFor
\end{algorithmic}
\end{algorithm}

Proposition~\ref{prop:cyclic-approximation} shows that the original nonstationary planning problem can be reduced, up to a controlled approximation error, to learning and optimizing the block-level proxy reward \(\bar r_m(\Ub)\).
Since Eq.~\eqref{eq:block-linear-reward} represents \(\bar r_m(\Ub)\) as a linear function of the unknown truncated Markov parameter \(\thetab_m^*\), we can treat each block \(\Ub\in\Ucal^{m+L}\) as a composite action in a linear bandit problem.
Based on this reduction, we propose $\algname$, an optimistic block-level algorithm described in Algorithm~\ref{alg:main-algorithm}.

The algorithm proceeds over blocks of length \(m+L\).
At the beginning of each block \(\tau\), it constructs a confidence set around the ridge estimator \(\hat{\thetab}_{\tau-1}\) using the Gram matrix \(\Vb_{\tau-1}\), and selects a block \(\Ub_\tau\) that maximizes a UCB on the proxy reward.
More specifically, the selected block maximizes the sum of the estimated reward and the uncertainty bonus over the informative suffix \(j=m+1,\ldots,m+L\):
\[  
    \mathrm{UCB}_{\tau-1}(\Ub)
    :=
    \sum_{j=m+1}^{m+L}
    \left(
        \langle
            \phib_j(\Ub),
            \hat{\thetab}_{\tau-1}
        \rangle
        +
        \gamma_{\tau-1}
        \|\phib_j(\Ub)\|_{\Vb_{\tau-1}^{-1}}
    \right).
\]
After playing the chosen block sequentially, the learner observes all rewards in the block, but uses only the last \(L\) rewards for the regression update.
The first \(m\) actions are used only to build the within-block history required by the finite-memory proxy, while the informative suffix provides noisy linear observations of \(\thetab_m^*\).

This design differs from the previous ETC approach~\citep{choi2026explorethencommit}.
Rather than spending an initial exploration phase to estimate the Markov parameters and then committing to a fixed action sequence, our proposed algorithm interleaves exploration and exploitation at the block level.
After each block, it updates its estimate of \(\thetab_m^*\) and immediately uses the refined confidence set to choose the next block optimistically.
This block-level adaptivity is the key mechanism that avoids the sub-optimal exploration--exploitation trade-off inherent in the previous ETC approach.

\begin{remark}[Maximization oracle] 
    The block selection step in Algorithm~\ref{alg:main-algorithm} is formulated as an offline UCB maximization problem over \(\Ucal^{m+L}\).
    Our theoretical analysis assumes access to an exact maximization oracle; similar oracle-based formulations have been widely used in the previous literature~\citep{chen2013combinatorial, kveton2015tight, clerici2024linear}.
    This allows us to focus on the statistical learning aspect of the problem, separately from the implementation of the offline optimizer.
    In practice, the UCB objective can be optimized approximately, for example by gradient ascent when \(\Ucal\) is continuous.
    Designing computationally efficient oracles tailored to the block-UCB objective, together with corresponding regret guarantees, is an interesting direction for future work.
\end{remark}

\subsection{Regret Bound}
We now state the main regret guarantee of Algorithm~\ref{alg:main-algorithm}.
Before stating the theorem, we introduce some notation.
For a fixed memory cutoff \(m \in \mathbb N\), let \(S_m\) be an upper bound on the norm of the vectorized truncated Markov parameter \(\thetab_m^*\), i.e., \(\|\thetab_m^*\|_2 \le S_m\).
Also, let \(L_\phi\) be a uniform upper bound on the lifted features, i.e., \(\|\phib_j(\Ub)\|_2 \le L_\phi\) for all \(\Ub\in\Ucal^{m+L}\) and \(j\in\{m+1,\ldots,m+L\}\).

\begin{restatable}[Pseudo-regret bound for Algorithm~\ref{alg:main-algorithm}]{lemma}{MainRegret}
\label{thm:main-regret}
    Suppose that Assumptions~\ref{assm:bounded-action}, \ref{assm:stability}, and~\ref{assm:noise} hold.
    Let \(d=p^2m\), and let \(N=T/(m+L)\) be an integer.
    For any \(\delta \in (0,1)\), set the parameters in Algorithm~\ref{alg:main-algorithm} as
    \(\lambda \ge L_\phi^2\) and \(\gamma_\tau=\alpha_\tau+\beta_\tau\), where
    \begin{align*}
        & \alpha_\tau
        :=
        \sigma_{\Bcal}
        \sqrt{
            d\log\!\left(
                1+
                \frac{\tau L L_\phi^2}{d\lambda}
            \right)
            +
            2\log\frac{1}{\delta}
        }
        +
        \sqrt{\lambda}S_m,
        \\
        & \beta_\tau
        :=
        \bar\Delta_m(\delta)
        \sqrt{
            2\tau Ld
            \log\!\left(
                1+
                \frac{\tau L L_\phi^2}{d\lambda}
            \right)
        } \, ,
        \quad 
        \sigma_{\Bcal}^2
        :=
        \sigma_z^2
        +
        \frac{
            \lambda_{\max}(\Sigmab_w)
            \|\Cb\|_2^2 M_\rho^2
        }{
            (1-\rho)^2
        },
        \\
        & \bar{\Delta}_m(\delta)
        :=
        \frac{
            \|\Cb\|_2 \|\Bb\|_2 M_\rho
        }{
            1-\rho
        }
        \rho^m
        +
        2 \rho^m
        \sqrt{
            \frac{
                2\lambda_{\max}(\Sigmab_w)
                \|\Cb\|_2^2M_\rho^2
            }{
                1-\rho^2
            }
            \log\frac{2\cdot 5^p T}{\delta}
        }
        .
    \end{align*}
    Then, with probability at least \(1-\delta\), the pseudo-regret of the policy \(\pi\) of Algorithm~\ref{alg:main-algorithm} satisfies
    \begin{equation*}
        \Rcal^{\mathrm{pseudo}}_T(\pi)
        \le
        2T\Delta_m
        +
        \frac{2Tm}{m+L}R_m
        +
        2\gamma_N
        \left[
            \frac{2LL_\phi d}{\sqrt{\lambda}}
            +
            \sqrt{
                2Td
                \log
                \left(
                    1+
                    \frac{T L_\phi^2}{\lambda d}
                \right)
            }
        \right] .
    \end{equation*}
    In particular, choosing \(m=-\frac12\log_\rho T\), 
    \(L=\frac{\sqrt{mT}}{d^{3/4}}-m\), and \(\lambda=L_\phi^2\) gives
    \[
        \Rcal^{\mathrm{pseudo}}_T(\pi)
        =
        \widetilde{\Ocal}\!\left(
            d^{3/4}\sqrt{mT}
            +
            d\sqrt T
        \right).
    \]
\end{restatable}

Now, we provide the actual expected regret guarantee of Algorithm~\ref{alg:main-algorithm}.
\begin{restatable}[Actual expected regret of Algorithm~\ref{alg:main-algorithm}]{theorem}{ActualRegret}
    \label{cor:actual-regret}
    Suppose the assumptions of Lemma~\ref{thm:main-regret} hold.
    For a fixed $\delta \in (0,1)$, let $B_T(\delta)$ be the pseudo-regret bound established in Lemma~\ref{thm:main-regret}.
    Define $\Delta_0 := \sum_{h=0}^\infty \| \Gb_h \|_2$.
    Then, we have
    \begin{equation*}
        \regret_T(\pi) \le B_T(\delta) + 2 \delta T \Delta_0 + \frac{\| \Cb \|_2 M_\rho}{1 - \rho} \sqrt{\frac{p \lambda_{\max}(\Sigmab_w)}{1 - \rho^2}} \frac{T}{m+L} \, .
    \end{equation*}
    Consequently, choosing \(\delta=T^{-2}\) and the parameters
    $
    m=-\frac{1}{2}\log_\rho T \, ,
    L=\frac{\sqrt{mT}}{d^{3/4}}-m \, ,
    \lambda=L_\phi^2 \, ,
    $
    up to integer rounding and whenever \(L\ge 1\), gives
    \[
    \regret_T(\pi)
    =
    \widetilde O\!\left(
    d^{3/4}\sqrt{mT}
    +
    d\sqrt T
    \right)
    = \tilde O\left(\left[\frac{p^{3/2}}{\log(1/\rho)^{5/4}} +\frac{p^2}{\log(1/\rho)}\right]\sqrt{T}\right)
    \, .
    \]
\end{restatable}

\paragraph{Discussion of Theorem~\ref{cor:actual-regret}.}
Theorem~\ref{cor:actual-regret} establishes an actual expected regret guarantee for Algorithm~\ref{alg:main-algorithm}.
Compared with the recent ETC approach for the same latent-dynamics bandit model~\citep{choi2026explorethencommit}, our result
improves the regret rate from \(\widetilde{\Ocal}(T^{2/3})\) to \(\widetilde{\Ocal}(\sqrt T)\) under the open-loop action-sequence benchmark.
The improvement comes from replacing the phase-separated explore-then-commit strategy with adaptive
block-level optimism. 
In ETC, the learner spends a long initial uniform-exploration phase to estimate the latent dynamics and then commits to a fixed open-loop sequence, which induces the suboptimal exploration--exploitation trade-off leading
to the \(\widetilde{\Ocal}(T^{2/3})\) rate. 
In contrast, Algorithm~\ref{alg:main-algorithm} updates its confidence set after each block and selects the next block optimistically, so every block contributes simultaneously to learning and reward maximization. Thus, adaptive exploration at the block level removes the need for a long pure-exploration phase and yields the improved
regret.

\subsection{Regret Analysis}
In this section, we present the main technical ingredients of the regret analysis for Lemma~\ref{thm:main-regret}.
First, we relate the pseudo-regret to the proxy regret for the block-level proxy reward.
Second, we establish a high-probability confidence set for the truncated Markov parameter \(\thetab_m^*\).
Finally, we use optimism to bound the proxy regret by a cumulative uncertainty term and control this term through a blockwise elliptical-potential argument.

\paragraph{Pseudo-regret decomposition.}
Let \(N:=  T/(m+L) \), and let \(\Ub_1,\ldots,\Ub_N\) be the sequence of blocks selected by Algorithm~\ref{alg:main-algorithm}.
Define the \textit{proxy regret}
$
    \Rcal_{\text{proxy}}
    :=
    \sum_{\tau=1}^N
    \left(
        \bar r_m(\Ub^*)
        -
        \bar r_m(\Ub_\tau)
    \right)\, ,
$
where \(\Ub^*\in\argmax_{\Ub\in\Ucal^{m+L}}\bar r_m(\Ub)\) is the proxy-optimal block.
By Proposition~\ref{prop:cyclic-approximation}, the pseudo-regret is bounded by the approximation error plus the proxy regret:
\begin{equation*}
    \Rcal^{\mathrm{pseudo}}_T(\pi)
    \le
    2T\Delta_m
    +
    \frac{2Tm}{m+L}R_m
    +
    \Rcal_{\text{proxy}} \, .
\end{equation*}
Thus, it remains to control \(\Rcal_{\text{proxy}}\).

\paragraph{Confidence set for truncated Markov parameters.}
We establish a confidence set for the truncated Markov parameter \(\thetab_m^*\).
The key step is to express the informative rewards in each block as noisy linear observations of \(\thetab_m^*\).
The following decomposition separates the observed reward into the linear signal, the current-block noise, the action-tail bias due to finite-memory truncation, and the process-noise tail coming from process noises outside the current finite-memory window.

\begin{restatable}[Reward decomposition]{lemma}{RewardDecomposition}
\label{lem:reward-decomposition}
    Let \(N = T/(m+L)\), and assume that \(N\in\mathbb N\).
    For each block \(\tau\in[N]\), let
    $\Ub_\tau
        =
        (\ub_{\tau,1},\ldots,\ub_{\tau,m+L})$,
    $\ub_{\tau,j}
        =
        \ub_{t_{\tau,j}}$,
    and
    $t_{\tau,j}:=(\tau-1)(m+L)+j$.
    For every \(j \in \{m+1,\ldots,m+L\}\), define
    \(r_{\tau,j}:=r_{t_{\tau,j}}\).
    Then,
    \begin{equation*}
        r_{\tau,j}
        =
        \left\langle
            \phib_j(\Ub_\tau),
            \thetab_m^*
        \right\rangle
        +
        \epsilon_{\tau,j}
        +
        \eta_{\tau,j}
        +
        \zeta_{\tau,j},
    \end{equation*}
    where
    \begin{align*}
        \epsilon_{\tau,j}
        &:=
        z_{t_{\tau,j}}
        +
        \sum_{h=0}^{j-2}
        \ub_{\tau,j}^\top
        \Cb \Ab^h
        \wb_{t_{\tau,j}-1-h} \, ,
        \tag{current-block noise}
        \\
        \eta_{\tau,j}
        &:=
        \sum_{s=m+1}^{\infty}
        \ub_{\tau,j}^\top
        \Gb_{s-1}
        \ub_{t_{\tau,j}-s} \, ,
        \tag{action-tail bias}
        \\
        \zeta_{\tau,j}
        &:=
        \sum_{h=j-1}^{\infty}
        \ub_{\tau,j}^\top
        \Cb \Ab^h
        \wb_{t_{\tau,j}-1-h} \, .
        \tag{process-noise tail}
    \end{align*}
\end{restatable}
Using the decomposition in Lemma~\ref{lem:reward-decomposition}, we construct a confidence set for the truncated Markov parameter \(\thetab_m^*\).
The current-block noise \(\epsilon_{\tau,j}\) is controlled through a self-normalized martingale concentration argument, whereas the action-tail bias \(\eta_{\tau,j}\) and the process-noise tail \(\zeta_{\tau,j}\) are controlled by the stability of the latent dynamics.
\begin{restatable}[Confidence set]{lemma}{ConfidenceSet}
\label{lem:confidence-set}
    Let \(N=T/(m+L)\), and assume that \(N\in\mathbb N\).
    Suppose Assumptions~\ref{assm:bounded-action}--\ref{assm:noise} hold.
    Fix \(\delta\in(0,1)\), and suppose that \(\lambda \ge L_\phi^2\).
    Then, with probability at least \(1-\delta\), for all \(\tau\in[N]\),
    $
        \|
            \hat{\thetab}_{\tau}
            -
            \thetab_m^*
        \|_{\Vb_\tau}
        \le
        \gamma_\tau ,
    $
    where \(\gamma_\tau\) is the confidence radius specified in Lemma~\ref{thm:main-regret}.
\end{restatable}
Lemma~\ref{lem:confidence-set} is the key ingredient that enables optimism in our block-level formulation.
At a high level, it plays a role analogous to the confidence set used in block-based analyses for linear bandits with finite memory~\citep{clerici2024linear}.
However, there is an important distinction.
In finite-memory linear bandits, the memory length is part of the model specification, and each informative reward is exactly linear in a finite-memory feature with martingale noise.

In our latent-dynamics model, however, the finite-memory representation is only an approximation of an infinite-memory reward process induced by the latent state dynamics.
This distinction is reflected in Lemma~\ref{lem:reward-decomposition}.
Each informative reward contains not only the current-block noise \(\epsilon_{\tau,j}\), but also the action-tail bias \(\eta_{\tau,j}\) and the process-noise tail \(\zeta_{\tau,j}\).
The term \(\eta_{\tau,j}\) appears because actions more than \(m\) steps in the past can still affect the current reward through the latent state, while \(\zeta_{\tau,j}\) captures the contribution of process noises generated before the current block.
Importantly, these pre-block noises can be correlated with the adaptively selected block through past reward observations, so \(\zeta_{\tau,j}\) is not a martingale-difference noise term with respect to the pre-block history.
Thus, unlike in the finite-memory setting, blocking does not yield a linear observation model with martingale noise.

We overcome these difficulties by exploiting the exponential stability of \(\Ab\).
The action-tail bias is bounded directly through the geometric decay of the Markov parameters.
For the process-noise tail, we separate the adaptively chosen action from the underlying vector-valued noise tail and establish a uniform high-probability bound over all informative indices.
This bound remains valid for the adaptively selected blocks.
The current-block noise is then controlled using a blockwise self-normalized martingale concentration argument.
Combining these bounds yields the confidence set in Lemma~\ref{lem:confidence-set}, extending the finite-memory block analysis to our latent-dynamics setting, where effective finite memory arises from stability rather than being built into the model.
   
\paragraph{Proxy regret.}
We control the proxy regret on the confidence event from Lemma~\ref{lem:confidence-set}.
By optimism, the proxy-optimal block is upper bounded by its UCB score, while Algorithm~\ref{alg:main-algorithm} selects a block maximizing this score. Hence, for every \(\tau\),
\[
    \bar r_m(\Ub^*)-\bar r_m(\Ub_\tau)
    \le
    2\gamma_{\tau-1}
    \sum_{j=m+1}^{m+L}
    \|\phib_j(\Ub_\tau)\|_{\Vb_{\tau-1}^{-1}} .
\]
Summing over \(\tau=1,\ldots,N\) and using the monotonicity of \(\gamma_\tau\), we obtain
\[
    \Rcal_{\mathrm{proxy}}
    \le
    2\gamma_N
    \sum_{\tau=1}^N
    \sum_{j=m+1}^{m+L}
    \|\phib_j(\Ub_\tau)\|_{\Vb_{\tau-1}^{-1}} .
\]
It remains to control the cumulative uncertainty term.
The main technical issue is that the UCB score uses the pre-block Gram matrix \(\Vb_{\tau-1}\), whereas the elliptical-potential argument naturally applies to sequentially updated within-block Gram matrices.
In the previous block-cyclic analysis~\citep{clerici2024linear}, this mismatch is absorbed by enlarging the confidence radius, which makes the block length appear multiplicatively in the leading estimation term.
Instead, we keep the confidence radius tied to \(\Vb_{\tau-1}\) and control the mismatch directly in the uncertainty sum.

Let \(\Vb_{\tau,j}:=\Vb_{\tau-1}+\sum_{i=m+1}^{j}\phib_i(\Ub_\tau)\phib_i(\Ub_\tau)^\top\), with \(\Vb_{\tau,m}:=\Vb_{\tau-1}\).
We decompose each uncertainty term into
\[
    \|\phib_j(\Ub_\tau)\|_{\Vb_{\tau-1}^{-1}}
    =
    \|\phib_j(\Ub_\tau)\|_{\Vb_{\tau,j-1}^{-1}}
    +
    \Big(
        \|\phib_j(\Ub_\tau)\|_{\Vb_{\tau-1}^{-1}}
        -
        \|\phib_j(\Ub_\tau)\|_{\Vb_{\tau,j-1}^{-1}}
    \Big).
\]
The first term is controlled by the usual elliptical-potential argument, and the second by a Gram-matrix telescoping argument. This yields
\[
    \sum_{\tau=1}^N
    \sum_{j=m+1}^{m+L}
    \|\phib_j(\Ub_\tau)\|_{\Vb_{\tau-1}^{-1}}
    \le
    \frac{2LL_\phi d}{\sqrt{\lambda}}
    +
    \sqrt{
        2Td
        \log
        \left(
            1+
            \frac{T L_\phi^2}{\lambda d}
        \right)
    } .
\]
Thus, \(L\) contributes only through an additive block-update mismatch term, rather than multiplying the leading \(\sqrt T\)-type uncertainty term.
Combining this bound with the proxy-regret decomposition proves Lemma~\ref{thm:main-regret}; the full derivation is deferred to Appendix~\ref{appx:main-theorem}.

\section{Experiments} \label{sec:experiments}
\begin{figure}[t!]
    \centering
    \includegraphics[width=1\linewidth]{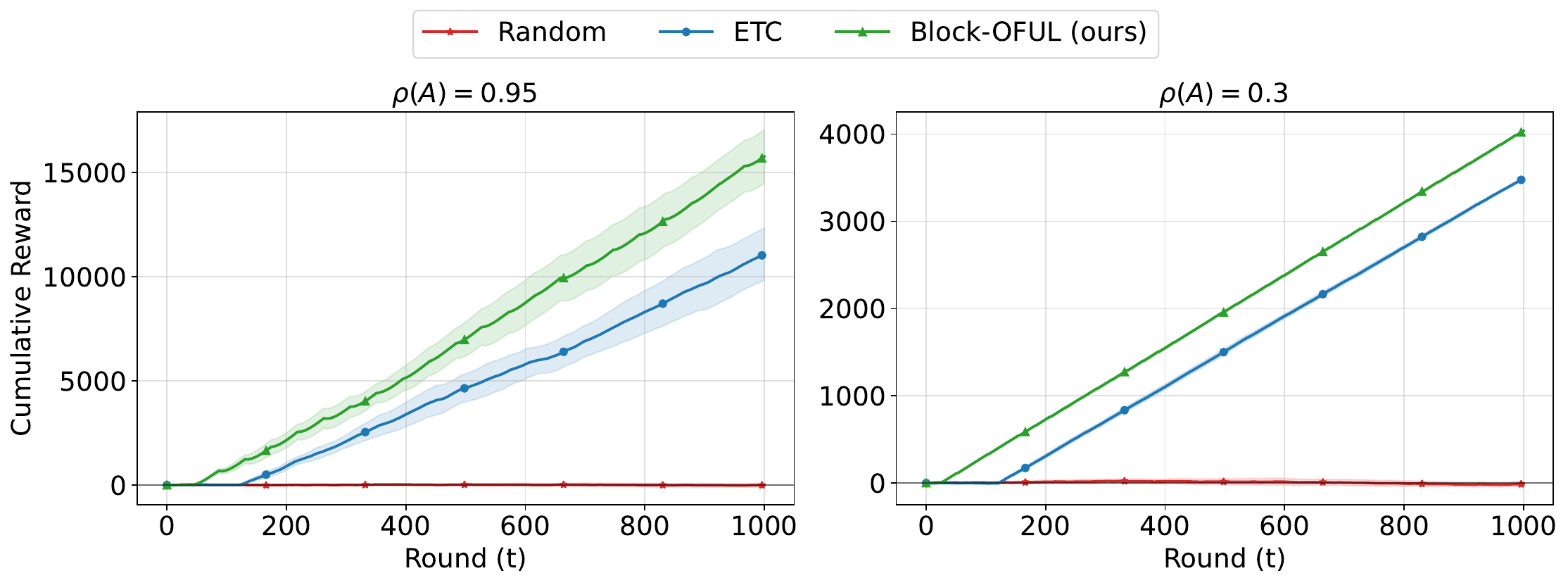}
    \caption{
    Cumulative reward on latent linear-dynamics bandit instances with different stability levels.
    }
    \label{fig:main-exp}
\end{figure}

We evaluate the proposed algorithm, \(\algname\), on two synthetic instances of the latent linear-dynamics model in Eq.~\eqref{eq:latent-dynamics}.
In both instances, the latent-state dimension is \(n=4\), the action dimension is \(p=3\), and the action set is
\(\Ucal=\{-1,+1\}^p\).
We use the same input and reward matrices
\[
\Bb=
\begin{bmatrix}
0.75 & 0.35 & 0.54\\
0.01 & 0.58 & 0.46\\
0.68 & -0.09 & -0.61\\
-0.15 & 0.83 & -0.47
\end{bmatrix},
\qquad
\Cb=
\begin{bmatrix}
0.67 & 0.04 & 0.72 & -0.19\\
0.20 & 0.65 & 0.00 & 0.76\\
0.54 & 0.46 & -0.61 & -0.47
\end{bmatrix},
\]
and consider two transition matrices with different stability levels:
\[
\begin{aligned}
\Ab_{0.95} &=
\begin{bmatrix}
0.52 & 0.14 & 0.27 & -0.18\\
0.14 & 0.50 & -0.08 & -0.06\\
0.27 & -0.08 & 0.75 & 0.03\\
-0.18 & -0.06 & 0.03 & 0.52
\end{bmatrix},
\qquad
\Ab_{0.3} &=
\begin{bmatrix}
0.16 & 0.04 & 0.09 & -0.06\\
0.04 & 0.13 & -0.02 & -0.03\\
0.09 & -0.02 & 0.23 & 0.00\\
-0.06 & -0.03 & 0.00 & 0.14
\end{bmatrix}.
\end{aligned}
\]
Here, \(\rho(\Ab_{0.95})=0.95\) and \(\rho(\Ab_{0.3})=0.3\).
These two instances allow us to examine the effect of the persistence of the latent dynamics.
We compare \(\algname\) with the ETC baseline~\citep{choi2026explorethencommit}.
For ETC, we use an exploration length of \(H=120\) in both instances.
For \(\algname\), we use \((m,L)=(8,36)\) when \(\rho(\Ab)=0.95\) and \((m,L)=(2,24)\) when \(\rho(\Ab)=0.3\).
The reward and process noises are Gaussian with variance proxy \(0.01\).

Figure~\ref{fig:main-exp} shows the cumulative reward over the interaction horizon.
Across both stability regimes, \(\algname\) consistently accumulates reward faster than ETC, while the random baseline remains close to zero.
For the instance with \(\rho(\Ab)=0.95\), the performance gap between \(\algname\) and ETC becomes increasingly pronounced over time, indicating that continued adaptation is particularly beneficial when the effects of past actions decay slowly.
For the faster-decaying instance with \(\rho(\Ab)=0.3\), both methods exhibit more stable reward accumulation, but \(\algname\) maintains a consistently higher cumulative-reward trajectory throughout most of the horizon.
Overall, these results demonstrate the benefit of adaptive block-level learning over the phase-separated ETC strategy across different levels of dynamical persistence.

\section{Conclusion}
We studied an endogenous nonstationary stochastic bandit problem with latent linear dynamics, where actions influence both immediate rewards and the future latent state.
This structure induces history-dependent rewards and makes long-horizon planning substantially more challenging than in classical stochastic bandits.
We showed that, under stable latent dynamics, the original infinite-memory problem can be reduced to a stationary block-level linear bandit problem through cyclic approximation.
Based on this reduction, we proposed an optimistic algorithm that learns the truncated Markov parameters while adaptively selecting blocks.
Our main result establishes a regret bound of order \(\widetilde{\Ocal}(\sqrt T)\), significantly improving over the previous \(\widetilde{\Ocal}(T^{2/3})\) explore-then-commit guarantee for the same model.
To the best of our knowledge, this is the first regret guarantee of order \(\widetilde{\Ocal}(\sqrt T)\) for latent linear-dynamics bandits with bilinear rewards and action-controlled latent states.
Future work includes developing computationally efficient oracles for block-UCB maximization, extending the analysis to adaptive or feedback benchmarks, and designing parameter-free variants that adapt to unknown stability and memory scales.

\clearpage
\bibliographystyle{plainnat}
\bibliography{references}

\clearpage
\appendix
\onecolumn
\renewcommand{\contentsname}{Contents of Appendix}
\addtocontents{toc}{\protect\setcounter{tocdepth}{3}}
{
  \hypersetup{hidelinks}
  \tableofcontents
}

\section{Proof of Proposition~\ref{prop:cyclic-approximation}} \label{appx:cyclic-approximation}
In this section, we provide the technical details for our cyclic approximation.
First, we show that under the stability condition, the truncation error decays geometrically.

\begin{lemma}[Finite-memory truncation error]
\label{lem:finite-memory-truncation-error}
Let \(\bar\mu_t(\ub_{1:t})\) denote the \(m\)-truncated version of the mean reward \(\mu_t(\ub_{1:t})\), defined by
\[
    \mu_t(\ub_{1:t})
    =
    \sum_{s=1}^{\infty}
    \ub_t^\top \Gb_{s-1}\ub_{t-s},
    \qquad
    \bar\mu_t(\ub_{1:t})
    :=
    \sum_{s=1}^{m}
    \ub_t^\top \Gb_{s-1}\ub_{t-s}.
\]
    
    Suppose Assumptions~\ref{assm:bounded-action} and~\ref{assm:stability} hold.
    Then, for every \(t \ge 1\), we have
    \begin{equation*}
        |\mu_t(\ub_{1:t})-\bar\mu_t(\ub_{1:t})|
        \le
        \Delta_m
        :=
        \sum_{h=m}^{\infty} \|\Gb_h\|_2
        \le
        \frac{\|\Cb\|_2 \| \Bb \|_2 M_\rho}{1-\rho}\rho^m,
    \end{equation*}
    and
    \begin{equation*}
        |\bar\mu_t(\ub_{1:t})|
        \le
        R_m
        :=
        \frac{\|\Cb\|_2 \|\Bb\|_2 M_\rho}{1-\rho}.
    \end{equation*}
\end{lemma}
\begin{proof}[Proof of Lemma~\ref{lem:finite-memory-truncation-error}]
    
    Recall that for every \(t \ge 1\), we have
    \begin{equation*}
        \mu_t(\ub_{1:t})
        :=
        \sum_{s=1}^{\infty}
        \ub_t^\top \Gb_{s-1}\ub_{t-s} \, .
    \end{equation*}
    For a fixed memory cutoff \(m \ge 1\), the truncated mean reward is defined as
    \begin{equation*}
        \bar\mu_t(\ub_{1:t})
        :=
        \sum_{s=1}^{m}
        \ub_t^\top \Gb_{s-1}\ub_{t-s} \, .
    \end{equation*}
    
    Then,
    \begin{equation*}
        \mu_t(\ub_{1:t}) - \bar\mu_t(\ub_{1:t})
        =
        \sum_{s=m+1}^{\infty}
        \ub_t^\top \Gb_{s-1} \ub_{t-s} \, .
    \end{equation*}
    Since \(\|\ub\|_2 \le 1\) for every \(\ub \in \Ucal\), we have
    \begin{equation*}
        \left|
            \ub_t^\top \Gb_{s-1} \ub_{t-s}
        \right|
        \le
        \|\ub_t\|_2
        \|\Gb_{s-1}\|_2
        \|\ub_{t-s}\|_2
        \le
        \|\Gb_{s-1}\|_2 \, .
    \end{equation*}
    Therefore,
    \begin{equation*}
        \left|
            \mu_t(\ub_{1:t}) - \bar\mu_t(\ub_{1:t})
        \right|
        \le
        \sum_{s=m+1}^{\infty}
        \|\Gb_{s-1}\|_2
        =
        \sum_{h=m}^{\infty}
        \|\Gb_h\|_2 \, .
    \end{equation*}
    Finally, by Assumption~\ref{assm:stability},
    \begin{equation*}
        \|\Gb_h\|_2
        =
        \|\Cb \Ab^h \Bb\|_2
        \le
        \|\Cb\|_2
        \|\Ab^h\|_2
        \|\Bb\|_2
        \le
        \|\Cb\|_2
        \|\Bb\|_2
        M_\rho \rho^h .
    \end{equation*}
    Hence, we have
    \begin{equation*}
        \left|
            \mu_t(\ub_{1:t}) - \bar\mu_t(\ub_{1:t})
        \right|
        \le
        \|\Cb\|_2\|\Bb\|_2M_\rho
        \sum_{h=m}^{\infty}\rho^h
        =
        \frac{
            \|\Cb\|_2\|\Bb\|_2M_\rho
        }{
            1-\rho
        }
        \rho^m \, .
    \end{equation*}
\end{proof}

Now, we provide the proof of Proposition~\ref{prop:cyclic-approximation}.
\CyclicApproximation*
\begin{proof}[Proof of Proposition~\ref{prop:cyclic-approximation}]
    Fix an arbitrary action sequence \(\ub_{1:T}\in\Ucal^T\), and partition it into consecutive blocks of length \(m+L\).
    By definition,
    \begin{align*}
        \sum_{t=1}^{T}\bar\mu_t(\ub_{1:t})
        &=
        \sum_{\tau=1}^{T/(m+L)}
        \sum_{j=1}^{m+L}
        \bar\mu_{(\tau-1)(m+L)+j}
        \big(
            \ub_{1:(\tau-1)(m+L)+j}
        \big)
        \\
        &=
        \sum_{\tau=1}^{T/(m+L)}
        \sum_{j=1}^{m}
        \bar\mu_{(\tau-1)(m+L)+j}
        \big(
            \ub_{1:(\tau-1)(m+L)+j}
        \big)
        \\
        &\quad
        +
        \sum_{\tau=1}^{T/(m+L)}
        \sum_{j=m+1}^{m+L}
        \bar\mu_{(\tau-1)(m+L)+j}
        \big(
            \ub_{1:(\tau-1)(m+L)+j}
        \big) \, .
    \end{align*}
    By Lemma~\ref{lem:finite-memory-truncation-error}, each prefix term satisfies
    \(|\bar\mu_t|\le R_m\). Hence,
    \begin{equation}
        \sum_{\tau=1}^{T/(m+L)}
        \sum_{j=1}^{m}
        \bar\mu_{(\tau-1)(m+L)+j}
        \big(
            \ub_{1:(\tau-1)(m+L)+j}
        \big)
        \le
        \frac{Tm}{m+L}R_m \, .
        \label{eq:prefix-bound-arbitrary}
    \end{equation}
    
    For the suffix term, since \(j\ge m+1\), the truncated reward depends only on the previous \(m\) actions within the same block.
    Therefore, for each block \(\tau\),
    \begin{align*}
        \sum_{j=m+1}^{m+L}
        \bar\mu_{(\tau-1)(m+L)+j}
        \big(
            \ub_{1:(\tau-1)(m+L)+j}
        \big)
        &=
        \bar r_m
        \big(
            \ub_{(\tau-1)(m+L)+1:\tau(m+L)}
        \big)
        \\
        &\le
        \bar r_m(\Ub^*),
    \end{align*}
    where the last inequality follows from the optimality of \(\Ub^*\in\argmax_{\Ub\in\Ucal^{m+L}}\bar r_m(\Ub)\).
    Combining this with Eq.~\eqref{eq:prefix-bound-arbitrary}, we obtain
    \begin{equation}
        \sum_{t=1}^{T}\bar\mu_t(\ub_{1:t})
        \le
        \frac{T}{m+L}\bar r_m(\Ub^*)
        +
        \frac{Tm}{m+L}R_m .
        \label{eq:mubar-upper-arbitrary}
    \end{equation}
    
    We now apply the same decomposition to the cyclic sequence \(\Ub^*_{\mathrm{cyc}}\).
    For each block, the suffix contribution is exactly \(\bar r_m(\Ub^*)\), while the prefix contribution is bounded from below by \(-mR_m\).
    Thus,
    \begin{align}
        \sum_{t=1}^{T}
        \bar\mu_t
        \big(
            (\Ub^*_{\mathrm{cyc}})_{1:t}
        \big)
        &\ge
        -\frac{Tm}{m+L}R_m
        +
        \frac{T}{m+L}\bar r_m(\Ub^*).
        \label{eq:mubar-lower-cyclic}
    \end{align}
    
    Combining Eqs.~\eqref{eq:mubar-upper-arbitrary} and~\eqref{eq:mubar-lower-cyclic}, we get
    \begin{equation*}
        \sum_{t=1}^{T}\bar\mu_t(\ub_{1:t})
        -
        \sum_{t=1}^{T}
        \bar\mu_t
        \big(
            (\Ub^*_{\mathrm{cyc}})_{1:t}
        \big)
        \le
        \frac{2Tm}{m+L}R_m \, .
    \end{equation*}
    Since this holds for every \(\ub_{1:T}\in\Ucal^T\), we have
    \begin{equation}
        \sup_{\ub_{1:T}\in\Ucal^T}
        \sum_{t=1}^{T}\bar\mu_t(\ub_{1:t})
        -
        \sum_{t=1}^{T}
        \bar\mu_t
        \big(
            (\Ub^*_{\mathrm{cyc}})_{1:t}
        \big)
        \le
        \frac{2Tm}{m+L}R_m .
        \label{eq:proxy-mubar-approx}
    \end{equation}
    
    It remains to transfer this bound from the truncated rewards to the original rewards.
    By Lemma~\ref{lem:finite-memory-truncation-error}, for every action sequence \(\ub_{1:T}\),
    \begin{equation*}
        \left|
            \sum_{t=1}^{T}\mu_t(\ub_{1:t})
            -
            \sum_{t=1}^{T}\bar\mu_t(\ub_{1:t})
        \right|
        \le
        \sum_{t=1}^{T}
        \left|
            \mu_t(\ub_{1:t})
            -
            \bar\mu_t(\ub_{1:t})
        \right|
        \le
        T\Delta_m \, .
    \end{equation*}
    Therefore,
    \begin{align*}
        \mathrm{OPT}_T
        -
        \sum_{t=1}^{T}
        \mu_t
        \big(
            (\Ub^*_{\mathrm{cyc}})_{1:t}
        \big)
        &\le
        \sup_{\ub_{1:T}\in\Ucal^T}
        \sum_{t=1}^{T}\bar\mu_t(\ub_{1:t})
        -
        \sum_{t=1}^{T}
        \bar\mu_t
        \big(
            (\Ub^*_{\mathrm{cyc}})_{1:t}
        \big)
        +
        2T\Delta_m
        \\
        &\le
        \frac{2Tm}{m+L}R_m
        +
        2T\Delta_m ,
    \end{align*}
    where the last inequality follows from Eq.~\eqref{eq:proxy-mubar-approx}.
\end{proof}

\section{Proof of Theorem~\ref{cor:actual-regret}} \label{appx:main-theorem}

\subsection{Proof of Lemma~\ref{lem:reward-decomposition}}
\RewardDecomposition*
\begin{proof}[Proof of Lemma~\ref{lem:reward-decomposition}]
    We first unroll the latent dynamics. Since \(\xb_1=\zero_n\), for any \(t\ge1\),
    \begin{align*}
        \xb_t
        &=
        \sum_{i=1}^{t-1}
        \Ab^{t-1-i}
        \left(
            \Bb\ub_i+\wb_i
        \right)  \\
        &=
        \sum_{i=1}^{t-1}
        \Ab^{t-1-i}\Bb\ub_i
        +
        \sum_{i=1}^{t-1}
        \Ab^{t-1-i}\wb_i .
    \end{align*}
    Therefore,
    \begin{align*}
        r_t
        &=
        \ub_t^\top \Cb\xb_t + z_t \\
        &=
        \sum_{i=1}^{t-1}
        \ub_t^\top
        \Cb\Ab^{t-1-i}\Bb\ub_i
        +
        \sum_{i=1}^{t-1}
        \ub_t^\top
        \Cb\Ab^{t-1-i}\wb_i
        +
        z_t .
    \end{align*}
    Reindexing the first sum by \(s=t-i\), and the second sum by \(h=t-1-i\), gives
    \begin{align*}
        r_t
        &=
        \sum_{s=1}^{\infty}
        \ub_t^\top
        \Gb_{s-1}
        \ub_{t-s}
        +
        \sum_{h=0}^{\infty}
        \ub_t^\top
        \Cb\Ab^h
        \wb_{t-1-h}
        +
        z_t \, ,
    \end{align*}
    where \(\Gb_h:=\Cb\Ab^h\Bb\), and we use the convention that \(\ub_s=\zero_p\) and \(\wb_s=\zero_n\) for \(s\le0\).
    
    Now fix a block \(\tau\in[N]\) and an informative position \(j\in\{m+1,\ldots,m+L\}\), and set \(t=t_{\tau,j}\).
    Since \(\ub_{t_{\tau,j}}=\ub_{\tau,j}\), the above identity yields
    \begin{align*}
        r_{\tau,j}
        &=
        \sum_{s=1}^{\infty}
        \ub_{\tau,j}^\top
        \Gb_{s-1}
        \ub_{t_{\tau,j}-s}
        +
        z_{t_{\tau,j}}
        +
        \sum_{h=0}^{\infty}
        \ub_{\tau,j}^\top
        \Cb\Ab^h
        \wb_{t_{\tau,j}-1-h}  \\
        &=
        \sum_{s=1}^{m}
        \ub_{\tau,j}^\top
        \Gb_{s-1}
        \ub_{t_{\tau,j}-s}
        +
        \sum_{s=m+1}^{\infty}
        \ub_{\tau,j}^\top
        \Gb_{s-1}
        \ub_{t_{\tau,j}-s} \\
        &\quad
        +
        z_{t_{\tau,j}}
        +
        \sum_{h=0}^{j-2}
        \ub_{\tau,j}^\top
        \Cb\Ab^h
        \wb_{t_{\tau,j}-1-h}
        +
        \sum_{h=j-1}^{\infty}
        \ub_{\tau,j}^\top
        \Cb\Ab^h
        \wb_{t_{\tau,j}-1-h} \, .
    \end{align*}
    Since \(j\ge m+1\), for every \(s=1,\ldots,m\), the action \(\ub_{t_{\tau,j}-s}\) lies in the same block and equals \(\ub_{\tau,j-s}\).
    Hence,
    \begin{align*}
        \sum_{s=1}^{m}
        \ub_{\tau,j}^\top
        \Gb_{s-1}
        \ub_{t_{\tau,j}-s}
        &=
        \sum_{s=1}^{m}
        \ub_{\tau,j}^\top
        \Gb_{s-1}
        \ub_{\tau,j-s}  \\
        &=
        \left\langle
            \phib_j(\Ub_\tau),
            \thetab_m^*
        \right\rangle ,
    \end{align*}
    by the definitions of \(\phib_j(\Ub_\tau)\) and \(\thetab_m^*\).
    Collecting the remaining terms according to the definitions of
    \(\epsilon_{\tau,j}\), \(\eta_{\tau,j}\), and \(\zeta_{\tau,j}\) gives
    \begin{equation*}
        r_{\tau,j}
        =
        \left\langle
            \phib_j(\Ub_\tau),
            \thetab_m^*
        \right\rangle
        +
        \epsilon_{\tau,j}
        +
        \eta_{\tau,j}
        +
        \zeta_{\tau,j} \, .
    \end{equation*}
\end{proof}

\subsection{Proof of Lemma~\ref{lem:confidence-set}}
\ConfidenceSet*

\begin{proof}[Proof of Lemma~\ref{lem:confidence-set}]

For each block \(s\in[N]\) and informative position \(j\in\{m+1,\ldots,m+L\}\), let
\[
    \phib_{s,j}:=\phib_j(\Ub_s).
\]
Define the block feature matrix and reward vector by
\[
    \Phib_s
    :=
    \big[
        \phib_{s,m+1},\ldots,\phib_{s,m+L}
    \big]
    \in \RR^{d\times L},
    \qquad
    \rb_s
    :=
    (r_{s,m+1},\ldots,r_{s,m+L})^\top
    \in \RR^L .
\]
Similarly, define
\[
    \epsilonb_s
    :=
    (\epsilon_{s,m+1},\ldots,\epsilon_{s,m+L})^\top,
    \quad
    \etab_s
    :=
    (\eta_{s,m+1},\ldots,\eta_{s,m+L})^\top,
    \quad
    \zetab_s
    :=
    (\zeta_{s,m+1},\ldots,\zeta_{s,m+L})^\top .
\]
Then Lemma~\ref{lem:reward-decomposition} gives the vector form
\[
    \rb_s
    =
    \Phib_s^\top \thetab_m^*
    +
    \epsilonb_s
    +
    \etab_s
    +
    \zetab_s .
\]

Recall that
\[
    \Vb_\tau
    =
    \lambda \Ib_d
    +
    \sum_{s=1}^{\tau}
    \Phib_s\Phib_s^\top,
    \qquad
    \widehat{\thetab}_\tau
    =
    \Vb_\tau^{-1}
    \sum_{s=1}^{\tau}
    \Phib_s \rb_s .
\]
Using the reward decomposition, we obtain
\begin{align*}
    \Vb_\tau
    \big(
        \widehat{\thetab}_\tau
        -
        \thetab_m^*
    \big)
    &=
    -\lambda \thetab_m^*
    +
    \sum_{s=1}^{\tau}
    \Phib_s \epsilonb_s
    +
    \sum_{s=1}^{\tau}
    \Phib_s(\etab_s+\zetab_s).
\end{align*}
Therefore,
\begin{equation}
    \left\|
        \widehat{\thetab}_\tau
        -
        \thetab_m^*
    \right\|_{\Vb_\tau}
    \le
    \sqrt{\lambda}\,
    \|\thetab_m^*\|_2
    +
    \underbrace{
        \left\|
            \sum_{s=1}^{\tau}
            \Phib_s\epsilonb_s
        \right\|_{\Vb_\tau^{-1}}
    }_{\text{block noise}}
    +
    \underbrace{
        \left\|
            \sum_{s=1}^{\tau}
            \Phib_s(\etab_s+\zetab_s)
        \right\|_{\Vb_\tau^{-1}}
    }_{\text{bias}}
    \, .
    \label{eq:confi_tri_decomp}
\end{equation}

The first term is bounded by \(\sqrt{\lambda}S_m\).
It remains to control the current-block noise term and the bias term.
For the current-block noise term, we introduce the following lemmas.
\begin{lemma}[Blockwise sub-Gaussianity of the current-block noise]
\label{lem:block_subgaussian}
    Let \(\Hcal_{\tau-1}\) denote the sigma-field just before the process and reward noises in block \(\tau\) are revealed; in particular, \(\Ub_\tau\) and \(\Phib_\tau\) are \(\Hcal_{\tau-1}\)-measurable.
    Under Assumptions~\ref{assm:bounded-action} and~\ref{assm:noise}, for every \(\tau\in[N]\), every \(\xi\in\RR^L\), and every \(\eta\in\RR\),
    \begin{equation*}
        \mathbb{E}\!\left[
            \exp\!\left(\eta \xi^\top \epsilonb_\tau\right)
            \,\middle|\,
            \Hcal_{\tau-1}
        \right]
        \le
        \exp\!\left(
            \frac{\eta^2\sigma_{\Bcal}^2}{2}
            \|\xi\|_2^2
        \right),
    \end{equation*}
    where
    \begin{equation*}
        \sigma_{\Bcal}^2
        :=
        \sigma_z^2
        +
        \lambda_{\max}(\Sigmab_w)
        \frac{\|\Cb\|_2^2 M_\rho^2}{(1-\rho)^2}.
    \end{equation*}
    Equivalently, \(\epsilonb_\tau\) is conditionally \(\sigma_{\Bcal}\)-sub-Gaussian as a vector in \(\RR^L\).
\end{lemma}

\begin{lemma}[Blockwise self-normalized bound] \label{lem:blockwise-self-normalized}
Let \((\Hcal_\tau)_{\tau\ge0}\) be a filtration.
Let \(\Phib_\tau\in\RR^{d\times L}\) be \(\Hcal_{\tau-1}\)-measurable, and let
\(\epsilonb_\tau\in\RR^L\) be \(\Hcal_\tau\)-measurable.
Assume that \(\epsilonb_\tau\) is conditionally \(\sigma\)-sub-Gaussian as a vector: for every
\(\xib\in\RR^L\) and every \(\alpha\in\RR\),
\[
    \mathbb E\!\left[
        \exp\!\left(
            \alpha \xib^\top \epsilonb_\tau
        \right)
        \,\middle|\,
        \Hcal_{\tau-1}
    \right]
    \le
    \exp\!\left(
        \frac{\alpha^2\sigma^2}{2}
        \|\xib\|_2^2
    \right).
\]
Let \(\lambda>0\), \(\Vb_0=\lambda\Ib_d\), and define
\[
    \Vb_\tau
    :=
    \Vb_0
    +
    \sum_{s=1}^{\tau}
    \Phib_s\Phib_s^\top,
    \qquad
    \Sbb_\tau
    :=
    \sum_{s=1}^{\tau}
    \Phib_s\epsilonb_s .
\]
Then, for any \(\delta\in(0,1)\), with probability at least \(1-\delta\), the following holds simultaneously for all \(\tau\ge0\):
\[
    \|\Sbb_\tau\|_{\Vb_\tau^{-1}}
    \le
    \sigma
    \sqrt{
        2\log
        \left(
            \frac{
                \det(\Vb_\tau)^{1/2}
            }{
                \det(\Vb_0)^{1/2}\delta
            }
        \right)
    } .
\]
\end{lemma}

Lemma~\ref{lem:block_subgaussian} shows that the vector of current-block noises
\(\epsilonb_\tau\) is conditionally sub-Gaussian at the block level.
This is needed because the \(L\) informative rewards within a block are not treated as independent scalar observations; rather, they enter the regression update jointly through the matrix \(\Phib_\tau\).
Lemma~\ref{lem:blockwise-self-normalized} then converts this blockwise sub-Gaussianity into a self-normalized concentration bound for the martingale sum
$
    \sum_{s=1}^{\tau}\Phib_s\epsilonb_s \, .
$
It is the block/vector-valued analogue of the standard self-normalized inequality for scalar martingale noise used in linear bandits~\citep{abbasi2011improved}.
Combining these two lemmas, the event
\begin{equation} \label{eq:confi_second_term}
    \Omega_{\mathrm{sn}}(\delta)
    :=
    \left\{
    \forall \tau\in[N]:
    \left\|
        \sum_{s=1}^{\tau}
        \Phib_s\epsilonb_s
    \right\|_{\Vb_\tau^{-1}}
    \le
    \sigma_{\Bcal}
    \sqrt{
        2\log
        \left(
            \frac{
                \det(\Vb_\tau)^{1/2}
            }{
                \det(\lambda\Ib_d)^{1/2}\delta
            }
        \right)
    }
    \right\}
\end{equation}
holds with probability at least \(1-\delta\).
Now it remains to control the bias term. We use the following high-probability bound for the action-tail bias and the process-noise tail.
\begin{lemma}[High-probability bias bound]
\label{lem:hp_noise_tail}
    Fix \(\delta\in(0,1)\). Then, with probability at least \(1-\delta\), the event
    \begin{equation*}
        \Omega_{\mathrm{bias}}(\delta)
        :=
        \left\{
            \forall \tau\in[N],\ 
            \forall j\in\{m+1,\ldots,m+L\},\ 
            |\eta_{\tau,j}+\zeta_{\tau,j}|
            \le
            \bar\Delta_m(\delta)
        \right\}
    \end{equation*}
    holds, where
    \begin{equation*}
        \bar{\Delta}_m(\delta)
        :=
        \frac{
            \|\Cb\|_2\|\Bb\|_2M_\rho
        }{
            1-\rho
        }
        \rho^m
        +
        2\rho^m
        \sqrt{
            2\lambda_{\max}(\Sigmab_w)\,
            \frac{\|\Cb\|_2^2M_\rho^2}{1-\rho^2}
            \log\frac{2\cdot 5^p T}{\delta}
        }
         \, .
    \end{equation*}
\end{lemma}

By Lemma~\ref{lem:hp_noise_tail}, the event
\(\Omega_{\mathrm{bias}}(\delta)\) holds with probability at least \(1-\delta\).
On \(\Omega_{\mathrm{bias}}(\delta)\), for every \(\tau\in[N]\),
\begin{align*}
    \left\|
        \sum_{s=1}^{\tau}
        \Phib_s(\etab_s+\zetab_s)
    \right\|_{\Vb_\tau^{-1}}
    &=
    \left\|
        \sum_{s=1}^{\tau}
        \sum_{j=m+1}^{m+L}
        \phib_{s,j}
        (\eta_{s,j}+\zeta_{s,j})
    \right\|_{\Vb_\tau^{-1}} \\
    &\le
    \sum_{s=1}^{\tau}
    \sum_{j=m+1}^{m+L}
    |\eta_{s,j}+\zeta_{s,j}|
    \|\phib_{s,j}\|_{\Vb_\tau^{-1}} \\
    &\le
    \bar\Delta_m(\delta)
    \sum_{s=1}^{\tau}
    \sum_{j=m+1}^{m+L}
    \|\phib_{s,j}\|_{\Vb_\tau^{-1}} \\
    &\le
    \bar\Delta_m(\delta)
    \sqrt{
        \tau L
        \sum_{s=1}^{\tau}
        \sum_{j=m+1}^{m+L}
        \|\phib_{s,j}\|_{\Vb_\tau^{-1}}^2
    },
\end{align*}
where the last inequality follows from Cauchy--Schwarz.

It remains to bound the sum of squared norms. 
Let \(\Vb_{s,j-1}\) denote the sequential Gram matrix immediately before adding the feature \(\phib_{s,j}\); that is,
\[
    \Vb_{s,j-1}
    :=
    \lambda \Ib_d
    +
    \sum_{q=1}^{s-1}
    \sum_{i=m+1}^{m+L}
    \phib_{q,i}\phib_{q,i}^\top
    +
    \sum_{i=m+1}^{j-1}
    \phib_{s,i}\phib_{s,i}^\top .
\]
Since \(\Vb_\tau\succeq \Vb_{s,j-1}\) for all \(s\le\tau\), we have
\[
    \|\phib_{s,j}\|_{\Vb_\tau^{-1}}
    \le
    \|\phib_{s,j}\|_{\Vb_{s,j-1}^{-1}} .
\]
Moreover, \(\lambda\ge L_\phi^2\) and \(\|\phib_{s,j}\|_2\le L_\phi\) imply
\[
    \|\phib_{s,j}\|_{\Vb_{s,j-1}^{-1}}^2
    \le
    \frac{\|\phib_{s,j}\|_2^2}{\lambda}
    \le 1 .
\]
Therefore, the standard elliptical-potential (Lemma~\ref{aux_lem:elliptical-potential}) bound gives
\[
    \sum_{s=1}^{\tau}
    \sum_{j=m+1}^{m+L}
    \|\phib_{s,j}\|_{\Vb_\tau^{-1}}^2
    \le
    \sum_{s=1}^{\tau}
    \sum_{j=m+1}^{m+L}
    \|\phib_{s,j}\|_{\Vb_{s,j-1}^{-1}}^2
    \le
    2\log\!\left(
        \frac{\det(\Vb_\tau)}
        {\det(\lambda\Ib_d)}
    \right).
\]
Consequently,
\begin{equation}
\label{eq:confi_third_term}
    \left\|
        \sum_{s=1}^{\tau}
        \Phib_s(\etab_s+\zetab_s)
    \right\|_{\Vb_\tau^{-1}}
    \le
    \bar\Delta_m(\delta)
    \sqrt{
        2\tau L
        \log\!\left(
            \frac{\det(\Vb_\tau)}
            {\det(\lambda\Ib_d)}
        \right)
    } .
\end{equation}

Now define the good event
\[
    \Omega_{\mathrm{good}}(\delta)
    :=
    \Omega_{\mathrm{sn}}(\delta/2)
    \cap
    \Omega_{\mathrm{bias}}(\delta/2).
\]
By the bounds above and a union bound,
\[
    \mathbb P\!\left(\Omega_{\mathrm{good}}(\delta)\right)
    \ge 1-\delta .
\]
On \(\Omega_{\mathrm{good}}(\delta)\), combining
Eqs.~\eqref{eq:confi_tri_decomp}, \eqref{eq:confi_second_term}, and~\eqref{eq:confi_third_term}, and using \(\|\thetab_m^*\|_2\le S_m\), yields for all \(\tau\in[N]\),
\begin{align*}
    \left\|
        \widehat{\thetab}_{\tau}
        -
        \thetab_m^*
    \right\|_{\Vb_\tau}
    \le
    &\;
    \sqrt{\lambda}S_m
    +
    \sigma_{\Bcal}
    \sqrt{
        2\log\!\left(
            \frac{
                2\det(\Vb_\tau)^{1/2}
            }{
                \det(\lambda \Ib_d)^{1/2}\delta
            }
        \right)
    } \\
    &\quad
    +
    \bar\Delta_m(\delta/2)
    \sqrt{
        2\tau L
        \log\!\left(
            \frac{\det(\Vb_\tau)}
            {\det(\lambda\Ib_d)}
        \right)
    } .
\end{align*}
The right-hand side is exactly the confidence radius \(\gamma_\tau\) chosen in Lemma~\ref{thm:main-regret}. Hence,
\[
    \left\|
        \widehat{\thetab}_{\tau}
        -
        \thetab_m^*
    \right\|_{\Vb_\tau}
    \le
    \gamma_\tau
    \qquad
    \forall \tau\in[N],
\]
with probability at least \(1-\delta\). 
\end{proof}

\subsection{Proof of Lemma~\ref{thm:main-regret}}
\MainRegret*

\begin{proof}[Proof of Lemma~\ref{thm:main-regret}]
    Let \(N:=T/(m+L)\), and assume \(N\in\mathbb N\).
    Define the proxy regret
    \[
        \Rcal_{\mathrm{proxy}}
        :=
        \sum_{\tau=1}^{N}
        \left(
            \bar r_m(\Ub^*)
            -
            \bar r_m(\Ub_\tau)
        \right).
    \]

    On the confidence event, for any block \(\Ub\in\Ucal^{m+L}\),
    \[
        \bar r_m(\Ub)
        \le
        \mathrm{UCB}_{\tau-1}(\Ub),
    \]
    where
    \[
        \mathrm{UCB}_{\tau-1}(\Ub)
        :=
        \sum_{j=m+1}^{m+L}
        \left(
            \left\langle
                \phib_j(\Ub),
                \widehat{\thetab}_{\tau-1}
            \right\rangle
            +
            \gamma_{\tau-1}
            \left\|
                \phib_j(\Ub)
            \right\|_{\Vb_{\tau-1}^{-1}}
        \right).
    \]
    Since \(\Ub_\tau\) maximizes this UCB score,
    \[
        \bar r_m(\widetilde{\bb})
        \le
        \mathrm{UCB}_{\tau-1}(\widetilde{\bb})
        \le
        \mathrm{UCB}_{\tau-1}(\Ub_\tau).
    \]
    Moreover, the confidence event also implies
    \[
        \bar r_m(\Ub_\tau)
        \ge
        \mathrm{UCB}_{\tau-1}(\Ub_\tau)
        -
        2\gamma_{\tau-1}
        \sum_{j=m+1}^{m+L}
        \left\|
            \phib_j(\Ub_\tau)
        \right\|_{\Vb_{\tau-1}^{-1}} .
    \]
    Combining the two inequalities gives
    \[
        \bar r_m(\widetilde{\bb})
        -
        \bar r_m(\Ub_\tau)
        \le
        2\gamma_{\tau-1}
        \sum_{j=m+1}^{m+L}
        \left\|
            \phib_j(\Ub_\tau)
        \right\|_{\Vb_{\tau-1}^{-1}} .
    \]
    Summing over \(\tau=1,\ldots,N\) and using the monotonicity of \(\gamma_\tau\), we obtain
    \begin{equation}
        \Rcal_{\mathrm{proxy}}
        \le
        2\gamma_N
        \sum_{\tau=1}^{N}
        \sum_{j=m+1}^{m+L}
        \left\|
            \phib_j(\Ub_\tau)
        \right\|_{\Vb_{\tau-1}^{-1}} .
        \label{eq:proxy-regret-ucb-1}
    \end{equation}
    
    It remains to control the cumulative uncertainty term.
    For each block \(\tau\), define the within-block sequential Gram matrices
    \[
        \Vb_{\tau,j}
        :=
        \Vb_{\tau-1}
        +
        \sum_{i=m+1}^{j}
        \phib_i(\Ub_\tau)\phib_i(\Ub_\tau)^\top,
        \qquad
        j=m+1,\ldots,m+L,
    \]
    with the convention \(\Vb_{\tau,m}:=\Vb_{\tau-1}\).
    Then
    \[
        \left\|
            \phib_j(\Ub_\tau)
        \right\|_{\Vb_{\tau-1}^{-1}}
        =
        \left\|
            \phib_j(\Ub_\tau)
        \right\|_{\Vb_{\tau,j-1}^{-1}}
        +
        \left(
            \left\|
                \phib_j(\Ub_\tau)
            \right\|_{\Vb_{\tau-1}^{-1}}
            -
            \left\|
                \phib_j(\Ub_\tau)
            \right\|_{\Vb_{\tau,j-1}^{-1}}
        \right).
    \]
    We bound the two terms separately.
    
    First, by Cauchy--Schwarz and the standard elliptical-potential lemma (Lemma~\ref{aux_lem:elliptical-potential}),
    \begin{align*}
        \sum_{\tau=1}^{N}
        \sum_{j=m+1}^{m+L}
        \left\|
            \phib_j(\Ub_\tau)
        \right\|_{\Vb_{\tau,j-1}^{-1}}
        &\le
        \sqrt{
            NL
            \sum_{\tau=1}^{N}
            \sum_{j=m+1}^{m+L}
            \left\|
                \phib_j(\Ub_\tau)
            \right\|_{\Vb_{\tau,j-1}^{-1}}^2
        } \\
        &\le
        \sqrt{
            2NL
            \log\!\left(
                \frac{\det(\Vb_N)}
                {\det(\lambda\Ib_d)}
            \right)
        } \\
        &\le
        \sqrt{
            2Td
            \log\!\left(
                1+
                \frac{T L_\phi^2}{\lambda d}
            \right)
        },
    \end{align*}
    where the last inequality uses \(NL\le T\), \(\|\phib_j(\Ub_\tau)\|_2\le L_\phi\), and the determinant-trace inequality (Lemma~\ref{aux_lem:det-tr inequality}).
    
    Next, since \(\Vb_{\tau,j-1}\preceq \Vb_\tau\), we have
    \[
        \left\|
            \phib_j(\Ub_\tau)
        \right\|_{\Vb_{\tau,j-1}^{-1}}
        \ge
        \left\|
            \phib_j(\Ub_\tau)
        \right\|_{\Vb_{\tau}^{-1}} .
    \]
    Hence,
    \[
        \left\|
            \phib_j(\Ub_\tau)
        \right\|_{\Vb_{\tau-1}^{-1}}
        -
        \left\|
            \phib_j(\Ub_\tau)
        \right\|_{\Vb_{\tau,j-1}^{-1}}
        \le
        \left\|
            \phib_j(\Ub_\tau)
        \right\|_{\Vb_{\tau-1}^{-1}}
        -
        \left\|
            \phib_j(\Ub_\tau)
        \right\|_{\Vb_{\tau}^{-1}} .
    \]
    Let \(\lambda_{\tau,1},\ldots,\lambda_{\tau,d}\) be the eigenvalues of \(\Vb_\tau\).
    To control the discrepancy between the pre-block and within-block Gram matrices, we use the Gram-matrix telescoping inequality for block updates (Lemma~\ref{aux_lem:spectral-inequality-Grammatrix}):
    \begin{align*}
    &\sum_{\tau=1}^{N}
    \sum_{j=m+1}^{m+L}
    \left(
        \left\|
            \phib_j(\Ub_\tau)
        \right\|_{\Vb_{\tau-1}^{-1}}
        -
        \left\|
            \phib_j(\Ub_\tau)
        \right\|_{\Vb_{\tau,j-1}^{-1}}
    \right) \\
    &\quad\le
    \sum_{\tau=1}^{N}
    \sum_{j=m+1}^{m+L}
    \left(
        \left\|
            \phib_j(\Ub_\tau)
        \right\|_{\Vb_{\tau-1}^{-1}}
        -
        \left\|
            \phib_j(\Ub_\tau)
        \right\|_{\Vb_{\tau}^{-1}}
    \right) \\
    &\quad\le
    \sum_{\tau=1}^{N}
    \sum_{j=m+1}^{m+L}
    2
    \left\|
        \phib_j(\Ub_\tau)
    \right\|_2
    \sum_{k=1}^{d}
    \left(
        \lambda_{\tau-1,k}^{-1/2}
        -
        \lambda_{\tau,k}^{-1/2}
    \right) \\
    &\quad\le
    2LL_\phi
    \sum_{\tau=1}^{N}
    \sum_{k=1}^{d}
    \left(
        \lambda_{\tau-1,k}^{-1/2}
        -
        \lambda_{\tau,k}^{-1/2}
    \right) \\
    &\quad\le
    2LL_\phi
    \sum_{k=1}^{d}
    \left(
        \lambda_{0,k}^{-1/2}
        -
        \lambda_{N,k}^{-1/2}
    \right) \\
    &\quad\le
    \frac{2LL_\phi d}{\sqrt{\lambda}},
    \end{align*}
    where the last step uses \(\Vb_0=\lambda\Ib_d\).
    
    Combining the two bounds gives
    \begin{equation}
        \sum_{\tau=1}^{N}
        \sum_{j=m+1}^{m+L}
        \left\|
            \phib_j(\Ub_\tau)
        \right\|_{\Vb_{\tau-1}^{-1}}
        \le
        \frac{2LL_\phi d}{\sqrt{\lambda}}
        +
        \sqrt{
            2Td
            \log\!\left(
                1+
                \frac{T L_\phi^2}{\lambda d}
            \right)
        } .
        \label{eq:blockwise-uncertainty-bound}
    \end{equation}
    Substituting Eq.~\eqref{eq:blockwise-uncertainty-bound} into Eq.~\eqref{eq:proxy-regret-ucb-1} yields
    \[
        \Rcal_{\mathrm{proxy}}
        \le
        2\gamma_N
        \left[
            \frac{2LL_\phi d}{\sqrt{\lambda}}
            +
            \sqrt{
                2Td
                \log\!\left(
                    1+
                    \frac{T L_\phi^2}{\lambda d}
                \right)
            }
        \right].
    \]
    
    Finally, we relate the proxy regret to the cumulative regret.
    Using the finite-memory truncation and block decomposition argument from Proposition~\ref{prop:cyclic-approximation}, for the action sequence generated by Algorithm~\ref{alg:main-algorithm},
    \[
        \Rcal^{\mathrm{pseudo}}_T(\pi)
        \le
        2T\Delta_m
        +
        \frac{2Tm}{m+L}R_m
        +
        \Rcal_{\mathrm{proxy}} \, .
    \]
    This concludes the first part of theorem.
    
    For the second part, we simplify the regret bound under a specific choice of \(m\) and \(L\).
    Under Assumption~\ref{assm:bounded-action}, the lifted feature satisfies
    \[
        \|\phib_j(\Ub)\|_2^2
        =
        \sum_{s=1}^{m}
        \|\ub_{j-s}\otimes \ub_j\|_2^2
        \le
        \sum_{s=1}^{m}
        \|\ub_{j-s}\|_2^2
        \|\ub_j\|_2^2
        \le m .
    \]
    Hence, we may take \(L_\phi\le \sqrt m\).
    Moreover, by Assumption~\ref{assm:stability},
    \[
        S_m^2
        =
        \|\thetab_m^*\|_2^2
        =
        \sum_{h=0}^{m-1}
        \|\Gb_h\|_F^2
        =
        \sum_{h=0}^{m-1}
        \|\Cb\Ab^h\Bb\|_F^2
        \le
        \|\Cb\|_2^2
        \|\Bb\|_F^2
        M_\rho^2
        \sum_{h=0}^{m-1}\rho^{2h}.
    \]
    Thus,
    \[
        S_m
        \le
        \frac{
            \|\Cb\|_2\|\Bb\|_F M_\rho
        }{
            \sqrt{1-\rho^2}
        },
    \]
    so \(S_m\) is uniformly bounded in \(m\) by a problem-dependent constant.
    
    Choose
    \[
        m=-\frac12\log_\rho T,
        \qquad
        L=\frac{\sqrt{mT}}{d^{3/4}}-m,
        \qquad
        d=p^2m,
    \]
    up to integer rounding, and set \(\lambda=L_\phi^2\).
    Then \(\rho^m=T^{-1/2}\). Since
    \[
        \Delta_m
        \le
        \frac{\|\Cb\|_2\|\Bb\|_2M_\rho}{1-\rho}\rho^m,
    \]
    we have
    \[
        T\Delta_m
        =
        \widetilde{\Ocal}(\sqrt T).
    \]
    Similarly, from the definition of \(\bar\Delta_m(\delta)\), we have
    \[
        \bar\Delta_m(\delta)
        =
        \widetilde{\Ocal}(T^{-1/2}).
    \]
    
    We next bound the confidence radius. Using \(NL\le T\), \(L_\phi\le\sqrt m\), \(\lambda=L_\phi^2\), and the uniform bound on \(S_m\), we obtain
    \[
        \gamma_N
        =
        \widetilde{\Ocal}
        \left(
            \sqrt d
            +
            \sqrt{\lambda}S_m
            +
            \bar\Delta_m(\delta)\sqrt{Td}
        \right)
        =
        \widetilde{\Ocal}(\sqrt d).
    \]
    Indeed, \(\sqrt{\lambda}S_m\le \sqrt m\,S_m\), which is dominated by \(\sqrt d=p\sqrt m\) up to problem-dependent constants, and
    \[
        \bar\Delta_m(\delta)\sqrt{Td}
        =
        \widetilde{\Ocal}(\sqrt d).
    \]
    
    The block-boundary term satisfies
    \[
        \frac{Tm}{m+L}R_m
        =
        \widetilde{\Ocal}
        \left(
            \frac{Tm}{L}
        \right)
        =
        \widetilde{\Ocal}
        \left(
            d^{3/4}\sqrt{mT}
        \right),
    \]
    where we used \(R_m=\Ocal(1)\) and \(L\asymp \sqrt{mT}/d^{3/4}\).
    
    Finally, the statistical term in Lemma~\ref{thm:main-regret} is bounded as
    \[
    \begin{aligned}
        \gamma_N
        \left[
            \frac{LL_\phi d}{\sqrt\lambda}
            +
            \sqrt{
                Td\log\left(
                    1+\frac{TL_\phi^2}{\lambda d}
                \right)
            }
        \right]
        &=
        \widetilde{\Ocal}
        \left(
            \sqrt d
            \left[
                Ld+\sqrt{Td}
            \right]
        \right)  \\
        &=
        \widetilde{\Ocal}
        \left(
            d^{3/4}\sqrt{mT}
            +
            d\sqrt T
        \right).
    \end{aligned}
    \]
    Combining the truncation error, the block-boundary error, and the statistical term yields
    \[
        \Rcal^{\mathrm{pseudo}}_T(\pi)
        =
        \widetilde{\Ocal}
        \left(
            d^{3/4}\sqrt{mT}
            +
            d\sqrt T
        \right).
    \]
    
    Since \(d=p^2m\) and
    $
        m
        =
        \frac{\log T}{2\log(1/\rho)},
    $
    the above bound can also be written, up to logarithmic factors in \(T\), as
    \[
        \Rcal^{\mathrm{pseudo}}_T(\pi)
        =
        \widetilde{\Ocal}
        \left(
            \left[
                \frac{p^{3/2}}{\{\log(1/\rho)\}^{5/4}}
                +
                \frac{p^2}{\log(1/\rho)}
            \right]
            \sqrt T
        \right),
    \]
    where \(\widetilde{\Ocal}\) hides logarithmic factors in \(T\) and \(\delta^{-1}\), as well as problem-dependent constants.
    In particular, for fixed \(p\) and fixed stability parameter \(\rho\), this gives
    \[
        \Rcal^{\mathrm{pseudo}}_T(\pi)=\widetilde{\Ocal}(\sqrt T).
    \]
\end{proof}

\subsection{Proof of Theorem~\ref{cor:actual-regret}}
\ActualRegret*
\begin{proof}[Proof of Theorem~\ref{cor:actual-regret}]

First, we introduce a lemma that shows, for block-committed policies, pseudo-regret and actual expected regret differ only by a boundary-noise term of order \(T/(m+L)\).

\begin{lemma}[Blockwise adaptation gap]
\label{lem:blockwise-adaptation-gap}
    Suppose Assumptions~1--3 hold, and assume for simplicity that
    \(N:=T/(m+L)\in\NN\).
    Consider any \((m+L)\)-block-committed policy \(\pi\), i.e., a policy that chooses the whole block
    \[
        \Ub_\tau=(\ub_{\tau,1},\ldots,\ub_{\tau,m+L})
    \]
    measurably with respect to
    $
        \Hcal_{\tau-1}:=\Fcal_{(\tau-1)(m+L)}
    $
    before any noise in block \(\tau\) is revealed, and then plays the block without within-block adaptation.
    Then
    \[
        \left|
        \regret_T(\pi)
        -
        \EE_\pi\!\left[\Rcal^{\mathrm{pseudo}}_T(\pi)\right]
        \right|
        \le
        \frac{\|\Cb\|_2 M_\rho}{1-\rho}
        \sqrt{
            \frac{p\,\lambda_{\max}(\Sigmab_w)}{1-\rho^2}
        } \frac{T}{m+L} \, .
    \]
\end{lemma}
Let \(\mathcal E^\mu_\delta\) be the high-probability event from
Lemma~\ref{thm:main-regret}. On \(\mathcal E^\mu_\delta\), we have
\[
    \Rcal^{\mathrm{pseudo}}_T(\pi) \le B_T(\delta).
\]
Moreover, since \(|\mu_t(\ub_{1:t})|\le \Delta_0\) for every action sequence,
\[
    0
    \le
    \Rcal^{\mathrm{pseudo}}_T(\pi)
    \le
    2T\Delta_0
\]
deterministically. Therefore,
\[
    \EE_\pi[\Rcal^{\mathrm{pseudo}}_T(\pi)]
    \le
    B_T(\delta)
    +
    2\delta T \Delta_0 .
\]
Combining this with Lemma~\ref{lem:blockwise-adaptation-gap} yields
\[
    \regret_T(\pi)
    \le
    B_T(\delta)
    +
    2\delta T \Delta_0
    +
    \frac{\|\Cb\|_2 M_\rho}{1-\rho}
    \sqrt{
        \frac{p\,\lambda_{\max}(\Sigmab_w)}{1-\rho^2}
    }
    \frac{T}{m+L}.
\]

For the second claim, observe that
\[
    \Delta_0
    :=
    \sum_{h=0}^\infty \|\Gb_h\|_2
    \le
    \frac{\|\Cb\|_2\|\Bb\|_2 M_\rho}{1-\rho}.
\]
The stated rate follows by taking \(\delta=T^{-2}\) and substituting the choices
of \(m\), \(L\), and \(\lambda\). In particular,
\[
    \frac{T}{m+L}
    \asymp
    \frac{d^{3/4}\sqrt T}{\sqrt m},
\]
so the additional adaptation-gap term is dominated by the leading
block-boundary and statistical terms up to logarithmic and problem-dependent factors.
\end{proof}

\subsection{Proof of Technical Lemmas}
\subsubsection{Proof of Lemma~\ref{lem:block_subgaussian}}
    \begin{proof}[Proof of Lemma~\ref{lem:block_subgaussian}]
    Fix \(\tau\in[N]\), \(\xib=(\xi_1,\ldots,\xi_L)^\top\in\RR^L\), and \(\alpha\in\RR\).
    For notational simplicity, write the informative position as \(j=m+\ell\), where \(\ell\in[L]\).
    By definition,
    \[
        \epsilon_{\tau,m+\ell}
        =
        z_{t_{\tau,m+\ell}}
        +
        \sum_{h=0}^{m+\ell-2}
        \ub_{\tau,m+\ell}^\top
        \Cb\Ab^h
        \wb_{t_{\tau,m+\ell}-1-h}.
    \]
    Hence,
    \[
        \xib^\top\epsilonb_\tau
        =
        \sum_{\ell=1}^L
        \xi_\ell z_{t_{\tau,m+\ell}}
        +
        \sum_{\ell=1}^L
        \sum_{h=0}^{m+\ell-2}
        \xi_\ell
        \ub_{\tau,m+\ell}^\top
        \Cb\Ab^h
        \wb_{t_{\tau,m+\ell}-1-h}.
    \]
    
    We reindex the process-noise terms by the within-block noise index \(q\in\{1,\ldots,m+L-1\}\), where the corresponding global index is \((\tau-1)(m+L)+q\).
    Then
    \[
        \xib^\top\epsilonb_\tau
        =
        \sum_{\ell=1}^L
        \xi_\ell z_{t_{\tau,m+\ell}}
        +
        \sum_{q=1}^{m+L-1}
        \left\langle
            \gb_{\tau,q}(\xib),
            \wb_{(\tau-1)(m+L)+q}
        \right\rangle,
    \]
    where
    \begin{equation*}
        \gb_{\tau,q}(\xib)
        :=
        \sum_{\ell=1}^L
        \mathbf 1\{q\le m+\ell-1\}
        \xi_\ell
        (\Ab^{m+\ell-1-q})^\top
        \Cb^\top
        \ub_{\tau,m+\ell}
        \in\RR^n \, .
    \end{equation*}
    
    Conditioned on \(\Hcal_{\tau-1}\), the actions in \(\Ub_\tau\) are fixed, and the noise variables in the current block are conditionally sub-Gaussian.
    The reward-noise part satisfies
    \begin{equation*}
        \mathbb E\!\left[
            \exp\!\left(
                \alpha
                \sum_{\ell=1}^L
                \xi_\ell z_{t_{\tau,m+\ell}}
            \right)
            \middle|
            \Hcal_{\tau-1}
        \right]
        \le
        \exp\!\left(
            \frac{\alpha^2\sigma_z^2}{2}
            \|\xib\|_2^2
        \right) \, .
    \end{equation*}
    For the process-noise part, Assumption~\ref{assm:noise} gives
    \[
        \mathbb E\!\left[
            \exp\!\left(
                \alpha
                \left\langle
                    \gb_{\tau,q}(\xib),
                    \wb_{(\tau-1)(m+L)+q}
                \right\rangle
            \right)
            \middle|
            \Hcal_{\tau-1}
        \right]
        \le
        \exp\!\left(
            \frac{\alpha^2}{2}
            \gb_{\tau,q}(\xib)^\top
            \Sigmab_w
            \gb_{\tau,q}(\xib)
        \right).
    \]
    Therefore, the process-noise contribution has variance proxy at most
    \[
        \lambda_{\max}(\Sigmab_w)
        \sum_{q=1}^{m+L-1}
        \|\gb_{\tau,q}(\xib)\|_2^2 .
    \]
    
    It remains to bound \(\sum_{q=1}^{m+L-1}\|\gb_{\tau,q}(\xib)\|_2^2\).
    Using \(\|\ub_{\tau,m+\ell}\|_2\le1\) and \(\|\Ab^h\|_2\le M_\rho\rho^h\), we obtain
    \begin{equation*}
        \|\gb_{\tau,q}(\xib)\|_2
        \le
        \sum_{\ell=1}^L
        \mathbf 1\{q\le m+\ell-1\}
        |\xi_\ell|
        \|\Cb\|_2
        M_\rho
        \rho^{m+\ell-1-q} \, .
    \end{equation*}
    This is bounded by the convolution of \((|\xi_\ell|)_{\ell=1}^L\) with the nonnegative sequence
    \[
        c_a := \|\Cb\|_2M_\rho\rho^a,
        \qquad a\ge0.
    \]
    By Young's convolution inequality,
    \[
        \sum_{q=1}^{m+L-1}
        \|\gb_{\tau,q}(\xib)\|_2^2
        \le
        \left(
            \sum_{a=0}^{\infty}
            c_a
        \right)^2
        \|\xib\|_2^2
        =
        \frac{\|\Cb\|_2^2M_\rho^2}{(1-\rho)^2}
        \|\xib\|_2^2 .
    \]
    Combining the reward-noise and process-noise bounds yields
    \[
        \mathbb{E}\!\left[
            \exp\!\left(\alpha \xib^\top \epsilonb_\tau\right)
            \,\middle|\,
            \Hcal_{\tau-1}
        \right]
        \le
        \exp\!\left(
            \frac{\alpha^2}{2}
            \left[
                \sigma_z^2
                +
                \lambda_{\max}(\Sigmab_w)
                \frac{\|\Cb\|_2^2 M_\rho^2}{(1-\rho)^2}
            \right]
            \|\xib\|_2^2
        \right) \, .
    \]
\end{proof}

\subsubsection{Proof of Lemma~\ref{lem:blockwise-self-normalized}}
\begin{proof}[Proof of Lemma~\ref{lem:blockwise-self-normalized}]
    Fix an arbitrary vector \(\vb\in\RR^d\).
    Since \(\Phib_\tau\) is \(\Hcal_{\tau-1}\)-measurable, the vector
    \(\Phib_\tau^\top\vb\in\RR^L\) is also \(\Hcal_{\tau-1}\)-measurable.
    By the conditional vector sub-Gaussianity assumption, applied with
    \(\xib=\Phib_\tau^\top\vb\), we have
    \[
        \mathbb E\!\left[
            \exp\!\left(
                \vb^\top\Phib_\tau\epsilonb_\tau
            \right)
            \,\middle|\,
            \Hcal_{\tau-1}
        \right]
        \le
        \exp\!\left(
            \frac{\sigma^2}{2}
            \|\Phib_\tau^\top\vb\|_2^2
        \right).
    \]
    Equivalently,
    \[
        \mathbb E\!\left[
            \exp\!\left(
                \vb^\top\Phib_\tau\epsilonb_\tau
                -
                \frac{\sigma^2}{2}
                \vb^\top\Phib_\tau\Phib_\tau^\top\vb
            \right)
            \,\middle|\,
            \Hcal_{\tau-1}
        \right]
        \le 1 .
    \]
    
    Define
    \[
        \Sbb_\tau
        :=
        \sum_{s=1}^{\tau}
        \Phib_s\epsilonb_s,
        \qquad
        \Vb_\tau
        :=
        \Vb_0+
        \sum_{s=1}^{\tau}
        \Phib_s\Phib_s^\top .
    \]
    For each fixed \(\vb\in\RR^d\), define
    \[
        M_\tau(\vb)
        :=
        \exp\!\left(
            \vb^\top\Sbb_\tau
            -
            \frac{\sigma^2}{2}
            \vb^\top
            (\Vb_\tau-\Vb_0)
            \vb
        \right).
    \]
    Then \((M_\tau(\vb))_{\tau\ge0}\) is a nonnegative supermartingale.
    Indeed,
    \[
    \begin{aligned}
        \mathbb E\!\left[
            M_\tau(\vb)
            \,\middle|\,
            \Hcal_{\tau-1}
        \right]
        &=
        M_{\tau-1}(\vb)
        \,
        \mathbb E\!\left[
            \exp\!\left(
                \vb^\top\Phib_\tau\epsilonb_\tau
                -
                \frac{\sigma^2}{2}
                \vb^\top\Phib_\tau\Phib_\tau^\top\vb
            \right)
            \,\middle|\,
            \Hcal_{\tau-1}
        \right] \\
        &\le
        M_{\tau-1}(\vb).
    \end{aligned}
    \]
    Moreover, \(M_0(\vb)=1\).
    
    We now use the standard mixture argument of \citet{abbasi2011improved}.
    Let \(\Lambdab\) be a Gaussian random vector, independent of all other randomness, with distribution
    \[
        \Lambdab
        \sim
        \mathcal N\!\left(
            \zero_d,\,
            \sigma^{-2}\Vb_0^{-1}
        \right),
    \]
    and let \(\mu\) denote its distribution.
    Define the mixture process
    \[
        M_\tau
        :=
        \mathbb E_{\Lambdab}
        \left[
            M_\tau(\Lambdab)
        \right]
        =
        \int_{\RR^d}
        M_\tau(\vb)\,\mu(d\vb).
    \]
    Since \(M_\tau(\vb)\ge0\), Fubini's theorem allows us to interchange the mixture integral and conditional expectation. Hence,
    \[
    \begin{aligned}
        \mathbb E[M_\tau\mid \Hcal_{\tau-1}]
        &=
        \mathbb E\!\left[
            \int_{\RR^d} M_\tau(\vb)\,\mu(d\vb)
            \,\middle|\,
            \Hcal_{\tau-1}
        \right] \\
        &=
        \int_{\RR^d}
        \mathbb E[
            M_\tau(\vb)\mid \Hcal_{\tau-1}
        ]
        \,\mu(d\vb) \\
        &\le
        \int_{\RR^d}
        M_{\tau-1}(\vb)\,\mu(d\vb)
        =
        M_{\tau-1}.
    \end{aligned}
    \]
    Therefore, \((M_\tau)_{\tau\ge0}\) is a nonnegative supermartingale.
    Also, since \(M_0(\vb)=1\) for every \(\vb\), we have \(M_0=1\).
    
    We compute \(M_\tau\) explicitly. The density of \(\Lambdab\) is
    \[
        f(\vb)
        =
        \frac{
            \det(\sigma^2\Vb_0)^{1/2}
        }{
            (2\pi)^{d/2}
        }
        \exp\!\left(
            -\frac{\sigma^2}{2}
            \vb^\top\Vb_0\vb
        \right).
    \]
    Hence,
    \[
    \begin{aligned}
        M_\tau
        &=
        \int_{\RR^d}
        \exp\!\left(
            \vb^\top\Sbb_\tau
            -
            \frac{\sigma^2}{2}
            \vb^\top(\Vb_\tau-\Vb_0)\vb
        \right)
        f(\vb)\,d\vb \\
        &=
        \frac{
            \det(\sigma^2\Vb_0)^{1/2}
        }{
            (2\pi)^{d/2}
        }
        \int_{\RR^d}
        \exp\!\left(
            \vb^\top\Sbb_\tau
            -
            \frac{\sigma^2}{2}
            \vb^\top\Vb_\tau\vb
        \right)
        d\vb .
    \end{aligned}
    \]
    Completing the square gives
    \[
        \vb^\top\Sbb_\tau
        -
        \frac{\sigma^2}{2}
        \vb^\top\Vb_\tau\vb
        =
        -\frac{\sigma^2}{2}
        \left\|
            \vb-\sigma^{-2}\Vb_\tau^{-1}\Sbb_\tau
        \right\|_{\Vb_\tau}^2
        +
        \frac{1}{2\sigma^2}
        \|\Sbb_\tau\|_{\Vb_\tau^{-1}}^2 .
    \]
    Therefore,
    \[
    \begin{aligned}
        M_\tau
        &=
        \frac{
            \det(\sigma^2\Vb_0)^{1/2}
        }{
            \det(\sigma^2\Vb_\tau)^{1/2}
        }
        \exp\!\left(
            \frac{1}{2\sigma^2}
            \|\Sbb_\tau\|_{\Vb_\tau^{-1}}^2
        \right) \\
        &=
        \frac{
            \det(\Vb_0)^{1/2}
        }{
            \det(\Vb_\tau)^{1/2}
        }
        \exp\!\left(
            \frac{1}{2\sigma^2}
            \|\Sbb_\tau\|_{\Vb_\tau^{-1}}^2
        \right).
    \end{aligned}
    \]
    
    By Ville's inequality for nonnegative supermartingales,
    \[
        \mathbb P\!\left(
            \exists \tau\ge0:
            M_\tau \ge \frac1\delta
        \right)
        \le \delta .
    \]
    Thus, with probability at least \(1-\delta\), for all \(\tau\ge0\),
    \[
        M_\tau < \frac1\delta .
    \]
    Using the explicit expression for \(M_\tau\), this implies
    \[
        \frac{
            \det(\Vb_0)^{1/2}
        }{
            \det(\Vb_\tau)^{1/2}
        }
        \exp\!\left(
            \frac{1}{2\sigma^2}
            \|\Sbb_\tau\|_{\Vb_\tau^{-1}}^2
        \right)
        <
        \frac1\delta .
    \]
    Equivalently,
    \[
        \|\Sbb_\tau\|_{\Vb_\tau^{-1}}^2
        \le
        2\sigma^2
        \log
        \left(
            \frac{
                \det(\Vb_\tau)^{1/2}
            }{
                \det(\Vb_0)^{1/2}\delta
            }
        \right).
    \]
    Taking square roots gives
    \[e
        \|\Sbb_\tau\|_{\Vb_\tau^{-1}}
        \le
        \sigma
        \sqrt{
            2
            \log
            \left(
                \frac{
                    \det(\Vb_\tau)^{1/2}
                }{
                    \det(\Vb_0)^{1/2}\delta
                }
            \right)
        } .
    \]
    This holds simultaneously for all \(\tau\ge0\), which completes the proof.
\end{proof}

\subsubsection{Proof of Lemma~\ref{lem:hp_noise_tail}}
    \begin{proof}[Proof of Lemma~\ref{lem:hp_noise_tail}]
    We first bound the action-tail bias \(\eta_{\tau,j}\).
    Since \(\|\ub_{\tau,j}\|_2\le1\) and \(\|\ub_{t_{\tau,j}-s}\|_2\le1\), we have
    \begin{align*}
        |\eta_{\tau,j}|
        &=
        \left|
            \sum_{s=m+1}^{\infty}
            \ub_{\tau,j}^\top
            \Gb_{s-1}
            \ub_{t_{\tau,j}-s}
        \right| \\
        &\le
        \sum_{s=m+1}^{\infty}
        \|\ub_{\tau,j}\|_2
        \|\Gb_{s-1}\|_2
        \|\ub_{t_{\tau,j}-s}\|_2 \\
        &\le
        \sum_{s=m+1}^{\infty}
        \|\Gb_{s-1}\|_2
        \\
        & =
        \sum_{h=m}^{\infty}
        \|\Gb_h\|_2
        \\
        & \le 
        \sum_{h=m}^{\infty}
        \|\Cb\|_2\|\Ab^h\|_2\|\Bb\|_2
        \\
        & \le 
        \frac{
            \|\Cb\|_2\|\Bb\|_2M_\rho
        }{
            1-\rho
        }
        \rho^m =: \Delta_m  \, .
    \end{align*}

    It remains to control the process-noise tail \(\zeta_{\tau,j}\).
    Because \(\ub_{\tau,j}\) may be chosen adaptively based on past rewards, we do not directly claim that \(\zeta_{\tau,j}\) is centered sub-Gaussian after conditioning on the selected action.
    Instead, define the vector-valued tail
    \[
        \yb_{\tau,j}
        :=
        \sum_{h=j-1}^{\infty}
        \Cb \Ab^h
        \wb_{t_{\tau,j}-1-h}
        \in \RR^p .
    \]
    Then, by \(\|\ub_{\tau,j}\|_2\le1\),
    \[
        |\zeta_{\tau,j}|
        =
        |\ub_{\tau,j}^\top \yb_{\tau,j}|
        \le
        \|\yb_{\tau,j}\|_2 .
    \]
    Hence, it suffices to control \(\|\yb_{\tau,j}\|_2\).
    
    Fix \((\tau,j)\). For any unit vector \(\vb\in\SSS^{p-1}\), we have
    \[
        \vb^\top \yb_{\tau,j}
        =
        \sum_{h=j-1}^{\infty}
        \vb^\top \Cb \Ab^h
        \wb_{t_{\tau,j}-1-h}.
    \]
    This is a centered sub-Gaussian random variable with variance proxy bounded by
    \begin{align*}
        \sum_{h=j-1}^{\infty}
        \lambda_{\max}(\Sigmab_w)
        \left\|
            (\Ab^h)^\top
            \Cb^\top
            \vb
        \right\|_2^2
        &\le
        \lambda_{\max}(\Sigmab_w)
        \sum_{h=j-1}^{\infty}
        \|\vb\|_2^2
        \|\Cb\|_2^2
        \|\Ab^h\|_2^2 \\
        &\le
        \lambda_{\max}(\Sigmab_w)
        \|\Cb\|_2^2 M_\rho^2
        \sum_{h=j-1}^{\infty}
        \rho^{2h}.
    \end{align*}
    Since \(j\ge m+1\), we have \(j-1\ge m\), and therefore
    \[
        \sum_{h=j-1}^{\infty}
        \rho^{2h}
        \le
        \sum_{h=m}^{\infty}
        \rho^{2h}
        =
        \frac{\rho^{2m}}{1-\rho^2}.
    \]
    Thus, every one-dimensional projection of \(\yb_{\tau,j}\) is sub-Gaussian with variance proxy at most
    \[
        \sigma_{\mathrm{tail}}^2
        :=
        \lambda_{\max}(\Sigmab_w)\,
        \frac{
            \|\Cb\|_2^2M_\rho^2
        }{
            1-\rho^2
        }
        \rho^{2m}.
    \]
    In particular, for each fixed \(\vb\in\SSS^{p-1}\) and every \(x>0\),
    \[
        \mathbb P
        \left(
            |\vb^\top \yb_{\tau,j}|>x
        \right)
        \le
        2\exp\!\left(
            -\frac{x^2}{2\sigma_{\mathrm{tail}}^2}
        \right).
    \]
    
    Let \(\Ncal_{1/2}\) be a \(1/2\)-net of \(\SSS^{p-1}\) with \(|\Ncal_{1/2}|\le 5^p\).
    Then,
    \[
        \|\yb_{\tau,j}\|_2
        \le
        2\max_{\vb\in\Ncal_{1/2}}
        |\vb^\top \yb_{\tau,j}| .
    \]
    Indeed, if \(\yb_{\tau,j}=0\), the claim is trivial.
    Otherwise, set
    \[
        \ub^* := \frac{\yb_{\tau,j}}{\|\yb_{\tau,j}\|_2}.
    \]
    By the definition of a \(1/2\)-net, there exists \(\vb\in\Ncal_{1/2}\) such that
    \(\|\ub^*-\vb\|_2\le 1/2\). Hence,
    \[
    \begin{aligned}
        \|\yb_{\tau,j}\|_2
        &=
        \langle \ub^*,\yb_{\tau,j}\rangle \\
        &=
        \langle \vb,\yb_{\tau,j}\rangle
        +
        \langle \ub^*-\vb,\yb_{\tau,j}\rangle \\
        &\le
        |\langle \vb,\yb_{\tau,j}\rangle|
        +
        \|\ub^*-\vb\|_2\|\yb_{\tau,j}\|_2 \\
        &\le
        \max_{\vb'\in\Ncal_{1/2}}
        |\langle \vb',\yb_{\tau,j}\rangle|
        +
        \frac12\|\yb_{\tau,j}\|_2 .
    \end{aligned}
    \]
    Rearranging gives the claim.
    
    Now, using the projection tail bound and a union bound over \(\Ncal_{1/2}\),
    \begin{align*}
        \mathbb P
        \left(
            \|\yb_{\tau,j}\|_2
            >
            \Delta'_m(\delta)
        \right)
        &\le
        \mathbb P
        \left(
            \max_{\vb\in\Ncal_{1/2}}
            |\vb^\top \yb_{\tau,j}|
            >
            \frac{\Delta'_m(\delta)}{2}
        \right) \\
        &\le
        \sum_{\vb\in\Ncal_{1/2}}
        \mathbb P
        \left(
            |\vb^\top \yb_{\tau,j}|
            >
            \frac{\Delta'_m(\delta)}{2}
        \right) \\
        &\le
        2|\Ncal_{1/2}|
        \exp\!\left(
            -
            \frac{(\Delta'_m(\delta))^2}
            {8\sigma_{\mathrm{tail}}^2}
        \right).
    \end{align*}
    Therefore, choosing
    \[
        \Delta'_m(\delta)
        :=
        2\sigma_{\mathrm{tail}}
        \sqrt{
            2\log\frac{2\cdot 5^p T}{\delta}
        }
    \]
    gives
    \[
        \mathbb P
        \left(
            \|\yb_{\tau,j}\|_2
            >
            \Delta'_m(\delta)
        \right)
        \le
        \frac{\delta}{T}.
    \]
    Since \(|\zeta_{\tau,j}|\le \|\yb_{\tau,j}\|_2\), we also have
    \[
        \mathbb P
        \left(
            |\zeta_{\tau,j}|
            >
            \Delta'_m(\delta)
        \right)
        \le
        \frac{\delta}{T}.
    \]
    
    Since the number of informative indices is
    \[
        NL
        =
        \frac{T}{m+L}L
        \le
        T,
    \]
    a union bound gives
    \[
        \mathbb P
        \left(
            \exists \tau\in[N],\ 
            \exists j\in\{m+1,\ldots,m+L\}:
            |\zeta_{\tau,j}|
            >
            \Delta'_m(\delta)
        \right)
        \le
        \delta.
    \]
    On the complement of this event, for all \(\tau\in[N]\) and \(j\in\{m+1,\ldots,m+L\}\),
    \[
        |\eta_{\tau,j}+\zeta_{\tau,j}|
        \le
        |\eta_{\tau,j}|+|\zeta_{\tau,j}|
        \le
        \Delta_m+\Delta'_m(\delta)
        =
        \bar\Delta_m(\delta).
    \]
\end{proof}

\subsubsection{Proof of Lemma~\ref{lem:blockwise-adaptation-gap}}
\begin{proof}[Proof of Lemma~\ref{lem:blockwise-adaptation-gap}]
    For any time \(t\), unrolling the latent dynamics gives
    \[
        r_t
        =
        \mu_t(\ub_{1:t})
        +
        z_t
        +
        \ub_t^\top \Cb
        \sum_{h=0}^{\infty}\Ab^h \wb_{t-1-h},
    \]
    where we use the convention \(\wb_s=\zero_n\) for \(s\le 0\).
    Since \(z_t\) is zero-mean and is revealed after \(\ub_t\) is chosen, \(\EE[z_t]=0\).
    Therefore,
    \[
        \regret_T(\pi)
        -
        \EE_\pi[\Rcal^{\mathrm{pseudo}}_T(\pi)]
        =
        -
        \EE_\pi
        \left[
            \sum_{t=1}^T
            \ub_t^\top \Cb
            \sum_{h=0}^{\infty}\Ab^h \wb_{t-1-h}
        \right].
    \]
    It remains to bound the absolute value of the right-hand side.
    
    Let \(t_{\tau,j}:=(\tau-1)(m+L)+j\), where
    \(\tau\in[N]\) and \(j\in[m+L]\). For the process-noise contribution at time
    \(t_{\tau,j}\), decompose
    \[
        \ub_{\tau,j}^\top \Cb
        \sum_{h=0}^{\infty}\Ab^h \wb_{t_{\tau,j}-1-h}
        =
        \underbrace{
        \ub_{\tau,j}^\top \Cb
        \sum_{h=0}^{j-2}\Ab^h \wb_{t_{\tau,j}-1-h}
        }_{=:a_{\tau,j}}
        +
        \underbrace{
        \ub_{\tau,j}^\top \Cb
        \sum_{h=j-1}^{\infty}\Ab^h \wb_{t_{\tau,j}-1-h}
        }_{=:b_{\tau,j}} .
    \]
    Here, \(a_{\tau,j}\) is the contribution of process noises generated within the current block,
    whereas \(b_{\tau,j}\) is the contribution of process noises generated before the current block.
    Because \(\Ub_\tau\) is \(\Hcal_{\tau-1}\)-measurable, all actions
    \(\ub_{\tau,j}\) are fixed before the noises in block \(\tau\) are revealed.
    By the martingale-difference property of the process noise and the tower property,
    \[
        \EE[a_{\tau,j}\mid \Hcal_{\tau-1}]=0,
        \qquad
        \text{and hence}
        \qquad
        \EE[a_{\tau,j}]=0.
    \]
    
    It remains to control \(b_{\tau,j}\).
    Let
    \[
        \yb_{\tau,j}
        :=
        \Cb\sum_{h=j-1}^{\infty}\Ab^h \wb_{t_{\tau,j}-1-h}.
    \]
    Then \(b_{\tau,j}=\ub_{\tau,j}^\top\yb_{\tau,j}\). Using \(\|\ub_{\tau,j}\|_2\le 1\),
    \[
        |\EE[b_{\tau,j}]|
        \le
        \EE\!\left[\|\yb_{\tau,j}\|_2\right].
    \]
    Since \(\wb_t\) is conditionally sub-Gaussian with variance proxy \(\Sigmab_w\),
    we have \(\EE[\wb_t\wb_t^\top\mid \Fcal_{t-1}]\preceq \Sigmab_w\).
    Moreover, cross terms between different time indices vanish by the martingale-difference property.
    Hence,
    \[
        \EE\!\left[\|\yb_{\tau,j}\|_2^2\right]
        \le
        \sum_{h=j-1}^{\infty}
        \EE\!\left[
            \|\Cb\Ab^h \wb_{t_{\tau,j}-1-h}\|_2^2
        \right].
    \]
    For each \(h\), by the covariance bound and Assumption~\ref{assm:stability},
    \[
        \EE\!\left[
            \|\Cb\Ab^h \wb_{t_{\tau,j}-1-h}\|_2^2
        \right]
        \le
        \lambda_{\max}(\Sigmab_w)\|\Cb\Ab^h\|_F^2
        \le
        p\,\lambda_{\max}(\Sigmab_w)\|\Cb\|_2^2M_\rho^2\rho^{2h}.
    \]
    Therefore,
    \[
        \EE\!\left[\|\yb_{\tau,j}\|_2^2\right]
        \le
        p\,\lambda_{\max}(\Sigmab_w)\|\Cb\|_2^2M_\rho^2
        \frac{\rho^{2(j-1)}}{1-\rho^2}.
    \]
    By Jensen's inequality,
    \[
        \EE\!\left[\|\yb_{\tau,j}\|_2\right]
        \le
        \|\Cb\|_2M_\rho
        \sqrt{
            \frac{p\,\lambda_{\max}(\Sigmab_w)}{1-\rho^2}
        }
        \rho^{j-1}.
    \]
    Summing over all positions in a block gives
    \[
        \sum_{j=1}^{m+L}
        |\EE[b_{\tau,j}]|
        \le
        \|\Cb\|_2M_\rho
        \sqrt{
            \frac{p\,\lambda_{\max}(\Sigmab_w)}{1-\rho^2}
        }
        \sum_{j=1}^{m+L}\rho^{j-1}
        \le
        \frac{\|\Cb\|_2M_\rho}{1-\rho}
        \sqrt{
            \frac{p\,\lambda_{\max}(\Sigmab_w)}{1-\rho^2}
        } .
    \]
    Finally, summing over the \(N=T/(m+L)\) blocks yields
    \[
        \left|
        \regret_T(\pi)
        -
        \EE_\pi[\Rcal^{\mathrm{pseudo}}_T(\pi)]
        \right|
        \le
        \frac{\|\Cb\|_2M_\rho}{1-\rho}
        \sqrt{
            \frac{p\,\lambda_{\max}(\Sigmab_w)}{1-\rho^2}
        }
        \frac{T}{m+L}.
    \]
\end{proof}

\section{Auxiliary Lemmas}
\begin{lemma}[Theorem 1 in~\citet{abbasi2011improved}] \label{aux_lem:self-normalized inequality}
    Let \(\{F_t\}_{t=0}^\infty\) be a filtration. Let \(\{\eta_t\}_{t=1}^\infty\) be a real-valued stochastic process such that \(\eta_t\) is \(F_t\)-measurable and \(\eta_t\) is conditionally \(R\)-sub-Gaussian for some \(R \ge 0\).
    Let \(\{X_t\}_{t=1}^\infty\) be an \(\mathbb{R}^d\)-valued stochastic process such that \(X_t\) is \(F_{t-1}\)-measurable. Assume that \(V\) is a \(d\times d\) positive definite matrix. For any \(t\ge 0\), define
    \[
    \overline V_t
    =
    V+\sum_{s=1}^t X_sX_s^\top,
    \qquad
    S_t
    =
    \sum_{s=1}^t \eta_s X_s.
    \]
    
    Then, for any \(\delta>0\), with probability at least \(1-\delta\), for all \(t\ge 0\),
    \[
    \|S_t\|_{\overline V_t^{-1}}^2
    \le
    2R^2\log\!\left(
    \frac{\det(\overline V_t)^{1/2}\det(V)^{-1/2}}{\delta}
    \right).
    \]
\end{lemma}

\begin{lemma}[Lemma 15 in~\citet{faury2020improved}] \label{aux_lem:elliptical-potential}
Let $\{x_s\}_{s=1}^{\infty}$ be a sequence in $\mathbb{R}^d$ such that
$\|x_s\|_2 \le X$ for all $s \in \mathbb{N}$, and let $\lambda$ be a non-negative scalar.
For $t \ge 1$, define
$
    V_t := \sum_{s=1}^{t-1} x_s x_s^\top + \lambda I_d \, .
$
The following inequality holds:
\[
    \sum_{t=1}^{T} \|x_t\|_{V_t^{-1}}^2
    \le
    2 \max\left(1, \frac{X^2}{\lambda}\right)
    \log \frac{\det(V_{T+1})}{\lambda^d} \, .
\]
\end{lemma}

\begin{lemma}[Lemma 10 in~\citet{abbasi2011improved}] \label{aux_lem:det-tr inequality}
    Suppose $X_1, \ldots, X_t \in \RR^d$ and for any $1 \le s \le t$, $\| X_s \|_2 \le L$,
    Let $\Vb_t = \lambda \Ib_d + \sum_{s=1}^t X_s X_s^\top$ for some $\lambda > 0$.
    Then,
    \begin{equation*}
        \det(\Vb_t) \le (\lambda + t L^2/d)^d \, .
    \end{equation*}
\end{lemma}

\begin{lemma}[Lemma 3 in~\citet{jin2021shrinking}] \label{aux_lem:spectral-inequality-Grammatrix}
Let $A \in \mathbb{R}^{d \times d}$ be any symmetric positive definite matrix,
and let $u_1, u_2, \ldots, u_L \in \mathbb{R}^d$ be any vectors.
Let $\lambda_1, \lambda_2, \ldots, \lambda_d$ be the eigenvalues of $A$,
and let $\nu_1, \nu_2, \ldots, \nu_d$ be the eigenvalues of
$A + \sum_{k=1}^{L} u_k u_k^\top$.
We have, for any $x \in \mathbb{R}^d$ such that $\|x\|_2 = 1$,
\[
    \sqrt{x^\top A^{-1}x}
    -
    \sqrt{
        x^\top
        \left(
            A + \sum_{k=1}^{L} u_k u_k^\top
        \right)^{-1}
        x
    }
    \le
    \sum_{i=1}^{d} \frac{2}{\sqrt{\lambda_i}}
    -
    \sum_{i=1}^{d} \frac{2}{\sqrt{\nu_i}} .
\]
\end{lemma}

\section{Experimental Details}
\label{app:experiment-details}
In this appendix, we provide details of the experiments reported in Section~\ref{sec:experiments}.
\subsection{Implementation of $\algname$}

For the continuous hypercube \(\Ucal=[-1,1]^p\), there exists an optimal open-loop action sequence whose actions all lie at vertices of the hypercube~\citep{choi2026explorethencommit}.
Motivated by this property and to enable a direct comparison with the ETC baseline, we restrict the experimental action set to
\(\Ucal=\{-1,+1\}^p\).

Even with this restriction, exhaustive maximization of the block-UCB objective over all
\(2^{p(m+L)}\) candidate blocks is computationally prohibitive.
We therefore implement the oracle using a multi-start local-search procedure.
For each restart, we initialize a block by sampling its actions uniformly from
\(\Ucal\), and repeatedly apply single-coordinate sign flips whenever they improve the block-UCB objective.
The local search terminates when no further improving flip is found.
We repeat this procedure from multiple random initializations and select the locally optimized block with the largest UCB objective among all restarts.

\subsection{Additional Experiments}

\paragraph{Environment.}
We further evaluate \(\algname\) on a structured delayed latent-dynamics environment designed to isolate the effects of delayed action influence and persistent latent memory.
The latent state is partitioned into \(D\) blocks of dimension \(p\), with \(n=pD\), and the action set is \(\Ucal=\{-1,+1\}^p\).
The system matrices are constructed so that each action is injected into the first latent block, propagated sequentially through the latent state with decay factor \(\rho\), and affects the reward only after reaching the final block:
\[
    \Ab
    =
    \rho
    \begin{bmatrix}
        \zero & \zero & \cdots & \zero \\
        \Ib_p & \zero & \cdots & \zero \\
        \zero & \Ib_p & \ddots & \vdots \\
        \vdots & \ddots & \ddots & \zero \\
        \zero & \cdots & \Ib_p & \zero
    \end{bmatrix},
    \qquad
    \Bb
    =
    \begin{bmatrix}
        \Ib_p\\
        \zero\\
        \vdots\\
        \zero
    \end{bmatrix},
    \qquad
    \Cb
    =
    \begin{bmatrix}
        \zero & \zero & \cdots & \Cb_D
    \end{bmatrix},
\]
where
\[
    \Cb_D
    =
    c_1\Ib_p
    +
    \frac{c_2}{p-1}
    \left(
        \mathbf 1\mathbf 1^\top-\Ib_p
    \right),
\]
with \(c_1=1\) and \(c_2=0.35\).
This construction creates an explicit delayed input--output dependence: increasing \(D\) delays the reward-relevant effect of an action, while increasing \(\rho\) makes this effect decay more slowly.
We use \(T=1000\), \(p=2\), \(D\in\{3,6\}\), and \(\rho\in\{0.3,0.95\}\), with Gaussian reward and process noise.

\begin{figure}[h]
    \centering
    \includegraphics[width=1\linewidth]{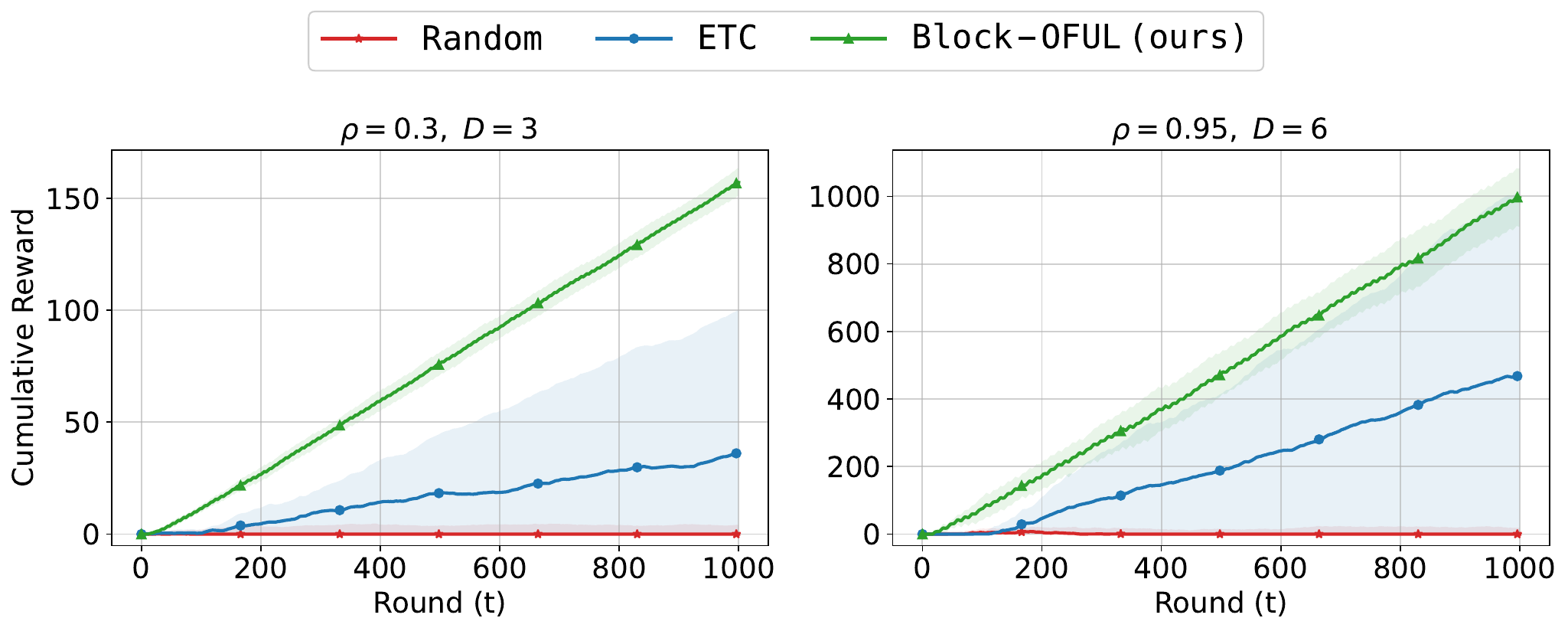}
    \caption{
    Cumulative reward on delayed latent linear-dynamics bandit instances with different levels of stability and delay.
    }
    \label{fig:additional-exp}
\end{figure}

Figure~\ref{fig:additional-exp} shows that \(\algname\) consistently achieves higher cumulative reward than ETC and the random baseline in the delayed latent-dynamics environments.
The advantage becomes more pronounced in the setting with larger \(D\) and \(\rho\) close to one, where useful reward information arrives only after a longer delay and the effects of past actions persist for longer.
This regime is particularly challenging for ETC because the quality of the committed action sequence depends entirely on the dynamics estimate obtained during its initial exploration phase.
In contrast, \(\algname\) continues to incorporate newly observed rewards and revise its block-level decisions throughout the horizon.
These results complement the main experiments by showing that the benefit of adaptive block-level optimism persists in environments with explicitly delayed and long-lasting action effects.

\end{document}